\documentclass{itor}
\usepackage{natbib}%
\usepackage[figuresright]{rotating}

\usepackage[margin=2.7cm]{geometry}

\usepackage[utf8]{inputenc}
\usepackage{amsmath}        
\usepackage{amsfonts}       
\usepackage{amsthm}         
\usepackage{bbding}         
\usepackage{bm}             
\usepackage{graphicx}       
\usepackage{fancyvrb}       
\usepackage{dcolumn}        
\usepackage{booktabs}       
\usepackage{paralist}       
\usepackage{algorithm}
\usepackage{algpseudocode}
\usepackage{subcaption}
\usepackage{textcomp}
\usepackage{mathtools}
\usepackage{siunitx}
\usepackage{pgfplotstable}
\usepackage{multirow}

\usepackage{lmodern}
\usepackage{anyfontsize}

\usepackage{tikz,pgfplots}
\usetikzlibrary{calc}
\pgfplotsset{compat=newest}

\theoremstyle{definition}

\theoremstyle{remark}

\jname{International Transactions in Operational Research}
\jvol{XX}
\jyear{20XX}
\doi{xx.xxxx/itor.xxxxx}

\begin{document}
\title{Purchase and Production Optimization in a Meat Processing Plant}

\author[M. Vlk, P. Šůcha, J. Rudy and R. Idzikowski]{Marek Vlk\affmark{a},  Přemysl Šůcha\affmark{a,$\ast$}, Jarosław Rudy\affmark{b} and Radosław Idzikowski\affmark{b}}

\affil{\affmark{a}Czech Technical University in Prague, Czech Institute of Informatics, Robotics and Cybernetics, Jugoslávských partyzánů 1580/3, 160 00 Prague 6 and Postcode, Czech Republic}
\affil{\affmark{b}Wrocław University of Science and Technology, Department of Control Systems and Mechatronics, Wybrzeże Wyspiańskiego 27, 50-370 Wrocław, Poland}
\email{Marek.Vlk@cvut.cz [M. Vlk]; Premysl.Sucha@cvut.cz [P. Šůcha];\\ jaroslaw.rudy@pwr.edu.pl [J. Rudy]; radoslaw.idzikowski@pwr.edu.pl[R. Idzikowski]}


\begin{abstract}
The food production industry, especially the meat production sector, faces many challenges that have even escalated due to the recent outbreak of the energy crisis in the European Union. Therefore, efficient use of input materials is an essential aspect affecting the profit of such companies. This paper addresses an optimization problem concerning the purchase and subsequent material processing we solved for a meat processing company. Unlike the majority of existing papers, we do not concentrate on how this problem concerns supply chain management, but we focus purely on the production stage. The problem involves the concept of alternative ways of material processing, stock of material with different expiration dates, and extra constraints widely neglected in the current literature, namely, the minimum order quantity and the minimum percentage in alternatives. We prove that each of these two constraints makes the problem \mbox{$\mathcal{NP}$-hard}, and hence we design a simple iterative approach based on integer linear programming that allows us to solve real-life instances even using an open-source integer linear programming solver. Another advantage of this approach is that it mitigates numerical issues, caused by the extensive range of data values, we experienced with a commercial solver. The results obtained using real data from the meat processing company showed that our algorithm can find the optimum solution in a few seconds for all considered use cases.
\end{abstract}

\thanks{\affmark{$\ast$}Author to whom all correspondence should be addressed (e-mail: Premysl.Sucha@cvut.cz).}
\historydate{Received DD MMMM YYYY; received in revised form DD MMMM YYYY; accepted DD MMMM YYYY}



\keywords{production; meat processing; alternatives; purchasing; problem complexity; integer linear programming}








\maketitle

\section{Introduction}


The food industry has been recently subjected to several challenges, such as the COVID-19 pandemic, CO$_2$ and other greenhouse gas emissions restraints~\citep{KHAN2024106707}, surging inflation, and energy prices. This, together with requirements regarding sustainability and lack of manpower, has brought great tension to the entire industry.
Meat production is a perfect example in the food industry that is affected by all the above-mentioned aspects. In addition, this sector is highly competitive, which motivates companies to reduce costs and optimize their processes. One of the options to face it is to optimize the use of materials (livestock, raw meat, etc.). This includes not only purchases of input raw materials but also the way materials are further processed in accordance with production demand. Domain-specific aspects, like several possible ways input raw materials can be transformed into different materials/products, make the optimization process hard, and thus, it cannot be done manually.

\textcolor{black}{In cooperation with a larger meat processing company, MÚÚÚ Písek, we address an integrated \textit{purchase and production optimization problem} (PPOP), intending to develop a~solution that can be applied to the existing production process in that company.} The decisions to be made include how many hogs to slaughter, how to cut and further process the carcasses, how much material to use from or to put in stock, and what to buy from other suppliers as already processed materials.
\textcolor{black}{
Meat stocks may be chilled or frozen, and, in addition, there may be multiple batches for each material with different expiration dates.}
Materials are always bought in bulks large enough to make it worthy for the suppliers to make the delivery. This means that the quantity of each material to be bought must be of the given required minimum value. This requirement is referred to as the \textit{minimum order quantity} (MOQ). Furthermore, the production process is described by a~large number of \textit{recipes} specifying the way livestock and raw meat are transformed into final products (in our case, we assumed 246 actively used recipes defined for 507 material types). There may be multiple ways to produce one product by using different sets of recipes. Moreover, a~recipe allows for defining groups of alternative materials that are mutually interchangeable. Additionally, if some material from a group of alternative materials is used in the recipe, it must be used in the given minimum relative quantity due to technological limitations. This limitation is because industrial production \textit{a priori} relies on large volumes of materials. Working with small quantities could cause problems, for example, in weighing, and would unnecessarily complicate the production process. We refer to this requirement as the \textit{minimum percentage in alternatives} (MPA).

\textcolor{black}{As will be seen from the literature review and its summary in Table~\ref{tab:literature}, the existing papers primarily concentrate on supply chains and do not account for many specific production properties of PPOP encountered in practice.
Hence, we limit our focus purely on PPOP as a single-stage problem. We investigate} its time complexity and show that each of the two requirements, MOQ and MPA, make the problem \mbox{$\mathcal{NP}$-hard}. 
Therefore, we design a model for the problem in the Linear Programming (LP) formalism, which omits the MOQ and MPA requirements, and a model using the Integer Linear Programming (ILP), which handles the problem with the MOQ and MPA requirements.
\textcolor{black}{The batches of a single material with different expiration dates are modeled using a piecewise linear objective function which allows a very compact ILP formulation.}
ILP formulation is a traditional approach to solving similar problems in the existing literature, as we show in Section~\ref{sec:related}. However, it turned out that ILP suffers from numerical instability when applied to real production data \textcolor{black}{(described later in Section~\ref{sec:scalability_demand})}. 
The instability is caused by the model size as well as the parameters of the recipes, ranging from decagrams to tons regarding material quantities.  
Hence, we propose a simple iterative approach where the \textcolor{black}{MOQ and MPA-related constraints} are added into the LP model only when necessary. \textcolor{black}{This strategy allows us to solve the problem using a non-commercial ILP solver and mitigate the issue with numerical stability.} The resulting solving method was tested together with the meat processing company, and it is currently deployed in the test operation as a part of their Enterprise Resource Planning (ERP) software.

The contributions of the paper are the following:
\begin{enumerate}[(i)]
    \item A mathematical model for the real-life problem, completed with new constraints, such as MOQ and alternatives with MPA.
    \item A complexity proof showing that the problem with either MOQ or MPA requirement itself is \mbox{$\mathcal{NP}$-complete} even when MOQ and MPA, respectively, is a constant equal for all materials.
    \item \textcolor{black}{A simple approach based on iterative constraint generation, which is suitable for large instances causing numerical instability and not requiring the use of commercial ILP solvers}.
    \item An experimental evaluation of our approach on real-life data from the meat processing company, including the scalability and the analysis of the impact of objective weights, the required number of hogs to slaughter, and MOQ.
\end{enumerate}

\textcolor{black}{The remainder of this paper is organized as follows. Section~\ref{sec:related} presents an~overview of the related literature. Section~\ref{sec:problemstatement} describes the optimization problem under consideration. The proposed solving method is presented in Section~\ref{sec:matmodeldesign}. The complexity of the problem is addressed in Section~\ref{sec:complexity}. Computational experiments and their analysis are given in Section~\ref{sec:experiments}.
Finally, Section~\ref{sec:conclusion} presents the conclusions of the paper.}

\section{Related Work}\label{sec:related}

The food industry is a vast field ranging from yogurt production~\citep[e.g.,][]{KOPANOS20112929} to beer production~\citep[e.g.,][]{BALDO201458}.
There is significant diversity among production processes in this field, with notable distinctions between them.
For instance, in meat production, a unique characteristic is the ability to derive several products from a single material (e.g., through the cutting of meat). This stands in contrast to other industries typically combining multiple materials to create a single product.
Another feature is that meat manufacturers produce many different products and assume alternatives in production processes, making production management challenging.
Due to this specificity, we will focus the literature review on meat production. 

The scientific literature regarding optimization in the meat industry is scarce, as indicated, for example, by \cite{mohebalizadehgashti2019multi}. Nonetheless, several approaches exist, and we will describe them briefly while showing that the problem we are considering in this paper is significantly different from existing works. This section is divided into four parts, describing (i) papers that are the most relevant to our (deterministic) approach, (ii) papers concerning meat production in more general supply-chain problems, (iii) papers describing meat production and supply chains under uncertainty, and (iv) papers concerning the MOQ constraint.

\subsection{Meat Processing}
We will start by describing the most relevant approaches to the studied problem. \cite{wikborg2008online} proposed an offline model for a beef plant in Norway to decide which cutting pattern to apply on different carcasses with penalization for unsatisfied demand. The paper considered 1000 products and 400 types of cutting. An online model was also proposed, with minimization of maximal regret, but it was restricted to one carcass. Local search and Genetic Algorithm heuristics were employed as well. Next, \cite{reynisdottir2012linear} considered a meat production supply chain focused on three levels: 1) cutting hogs into parts (e.g., side or leg), 2) processing of parts into raw materials (e.g., raw ham), and 3) processing raw materials into products (e.g., Bayonne ham). The author considered the cost of holding inventory as well as the possibility of buying additional products on levels 1 and 3. The objective function was to maximize the value of produced goods while minimizing the resulting inventory and buying costs. The resulting ILP model was compared with the actual company data, but the results were limited due to multiple model simplifications and data discrepancies. 
A similar approach but with a more complex model was considered by \cite{albornoz2015optimization}. Additional resources and processes such as product freezing, regular and overtime working hours, as well as warehouse capacities, were considered. Moreover, various cutting patterns for several types of hog carcasses were included, but there was no option to buy products from vendors. \textcolor{black}{While frozen products were considered, the freezing was performed every time a given product was too old. As such, this approach does not allow to explicitly decide which products should be frozen.}
Experiments carried out in CPLEX assuming 17 cutting patterns and 40 products for 10-day planning showed an 8\% profit increase compared to the decision-making employed by the company.
This paper was later extended by \cite{rodriguezcoordination}, who used a static Stackelberg game to formulate a two-level LP. The game consisted of two players: the producer (who processes carcasses into products) and the retailer (who distributes and markets the products). The game proceeds in a sequential and uncooperative manner. First, the producer decides on its production quantity, which becomes known to the retailer in the next turn. If the products supplied by the producer do not cover the demand of the retailer, the retailer imposes a cost penalty on the producer in the subsequent turn. The game continues with both sides trying to maximize their objective functions. However, aside from the model formulation, the paper contained no experiment section, as the research was still ongoing at the time of publication.

To the best of our knowledge, the above four papers are the closest to the problem considered in this paper. However, they do not consider MPA or MOQ in any way. Moreover, only \cite{wikborg2008online} considered large-scale problem instances with hundreds of products. Apart from that, none of the papers study the complexity of the problem or the numerical stability of the ILP-based solutions.
Next, we will review the remaining papers that address deterministic meat- or food-related supply chains or similar problems.

\subsection{Meat Production Models Integrated with Supply chains}
\textcolor{black}{Meat processing is often solved within supply chains; however, those papers often simplify the meat production phase and concentrate on a larger scope. For example, }
a typical three-stage meat supply chain with multiple objectives was considered by \cite{mohammed2017fuzzy}. Four objectives were used: 1) production costs, 2) $\mathrm{CO}_2$ emission, 3) distribution time, and 4) delivery rate. The authors considered three solving methods: 1) LP-metrics, 2) the $\epsilon$-constraint, and 3) goal programming. \cite{mohammed2017multi} considered a similar three-tier meat production supply chain and defined an objective function composed of operational costs (including the cost of radio-frequency identification tags), customer satisfaction, and meat quality. Again, the problem is considered on a higher level (number of employees at each location, transport time, and distance) instead of a lower level (details of meat cutting and product recipes). \cite{blanco2020} considered a deterministic variant of the broiler chicken production problem. The approach is, again, different from ours as the author focuses mainly on the management of chicken growth (and the inventory and cost issues associated with it). The slaughterhouses are mentioned as the last stage, but no details, aside from the available storage and mortality rates, are considered. It should be noted, however, that, to the best of our knowledge, this paper is the only example of meat-related research that mentions MOQ as there is a minimum number of day-old chickens that should arrive at the farm each week. The author then proposes an ILP model and confronts it with a poultry supply chain in Colombia.
Broiler production planning in a supply chain is also studied by \cite{Brevik2020broiler}. The authors studied this process in a Norwegian broiler production company governing approximately 150 broiler farms. Similar to the previous work, the authors concentrate on the management of chicken growth while slaughtering is the last stage of the studied processes, and further meat processing is not assumed. The problem is formulated using ILP; nevertheless, due to the size of the problem, the authors propose two different rolling horizon heuristics. The authors show interesting sensitivity analyses providing insights regarding certain strategic decisions.
\cite{Arash2017fish} address the fish supply chain. The authors developed a methodology to help farmers make decisions on spawn purchase quantity, time to harvest fish from water, and farming periods. The problem is formulated as ILP and the objective is to maximize the total profit.
An example of a meat supply chain with multiple objectives, including costs and carbon emission, was considered by \cite{schmidt2022traceability}. The problem was modeled using ILP and solved using goal programming. However, instead of considering just a meat processing plant, the paper takes a more general approach, focusing on the flow of materials from slaughterhouses through processing plants to retailers. On the other hand, specific cutting patterns of carcasses are not considered.
\cite{koroteev2022optimization} considered the problem of budget optimization and searched for an optimal production scheme to increase the performance of a meat processing enterprise. The authors employ a Monte Carlo method and use real-life data to evaluate their solution. Nonetheless, again, no recipe alternatives or MOQ constraints were used. Moreover, much of the paper focuses on forecasts and the Monte Carlo method rather than the model itself.  On the other hand, the considered case study was large, with over 600 products, raw materials, and other items, as well as over 5000 operations encountered in the production process. A~similar approach for meat production was shown in the paper by \cite{https://doi.org/10.1111/itor.13014}. The authors considered optimization of managing beef herds for small-scale farmers in Czech Republic, with heifer acquisition and selection for breeding and fattening as the decision variables. Using 102 scenarios, the results showed the impact of heifer acquisition and subsidies size on the sustainability of business, indicating a~threshold under which the best decision is to leave the market. Finally, \cite{doi:10.1080/01605682.2019.1640588} considered a~lot sizing and scheduling problem with scarce resources and setup times for a~Brazilian meat company. As production resources are scarce, the solution has to decide which production lines have to be assembled, impacting the solution cost. \textcolor{black}{The approach also explicitly accounted for products perishability and incurred a~percentage loss when selling aged products.} The authors considered a~MIP formulation and a~reformulation into the Branch-and-Bound method and carried out extensive experiments on a~number of test cases.

\subsection{Meat Production Models Assuming Uncertainty}
We will now move on to the papers that consider meat production and supply chains under uncertainty (usually meaning uncertain demand). Since \textcolor{black}{the uncertain aspect was not crucial for the company, we will focus on the problems} and models that are studied by the papers in this area.

We will start with the paper by \cite{cooperative2005supply} as it appears the closest to our work. The authors consider the meat production process for the Norwegian Meat Cooperative with four stages: 1) slaughtering, 2) cutting, 3) processing, and 4) distribution/sales. The paper describes both cutting patterns in the cutting stage and recipes in the processing stage. Most importantly, it is mentioned that a product can have multiple recipes (as many as 20). This allows one to choose the most cost-efficient recipe depending on the current costs of ingredients as well as geographical region. This 'multiple recipes per product' can be viewed as a sort of alternative; however, it does not seem to match the 'multiple ingredient types for a given recipe input' we are considering in this paper (see recipe 3 in Figure~\ref{fig:example}). Finally, the paper by Tomasgard and Høeg does not consider MOQ. A similar case study in the paper by~\cite{schutz2009supply} addresses the Norwegian meat industry. The authors considered a stochastic supply chain formulation with the goal of minimizing annual costs. A four-stage production process, including cutting and processing stages, is formulated with nearly 300 recipes (including the cutting stage) and over 100 final products. It should also be noted that this is one of the rare papers that explicitly mentions the concept of bills of materials. The proposed mathematical formulation is more closely related to our approach as recipes and bills of materials are included in it. On the other hand, the model also considers locations and transport times but does not consider MOQ or recipe alternatives. A different approach was employed by \cite{rijpkema2016application}, where the authors studied the allocation of animals to slaughterhouses in a meat supply chain for a large European meat processor. The research aimed to reduce the quality uncertainty at the slaughterhouses compared to the approach currently employed at the company, which was based on minimizing transportation time between slaughterhouses. The proposed ILP considers, again, a more general approach with no emphasis on details of cutting and processing of meat. A job assignment problem for the chicken parts cutting process was considered by \cite{leksakul2019fuzzifying}. The authors proposed an ILP formulation to maximize the revenue over the weekly planning horizon. The results brought a significant increase in revenue and throughput while also reducing backorders. While the paper considers a few special cutting types, it mostly disregards the other production stages and aims to employ fuzzification to handle uncertain demand. \textcolor{black}{\cite{arabsheybani2024sustainable} proposed a~robust optimization approach for a~multi-item multi-period and multi-tier livestock supply chain. In search for improved sustainability, the authors considered both different ages of products in stock and possibility of choosing how much of the product was frozen. Similarly, \cite{FATHOLLAHZADEH2024110578} considered multi-objective optimization for a~green meat supply chain under uncertainty with waste and by-products utilization. The paper modeled a~supply chain where meat products can be used in one of three ways: fresh, frozen and processed. This allows to choose if a~product will be frozen, even if it is somewhat different from our notion of chilling and freezing recipes.} Finally, in their Master thesis, \cite{mohebalizadehgashti2019multi} considered a multi-objective ILP formulation for a multi-product, multi-tier green meat supply chain with a case study in Canada. The approach, however, is focused more on the general problem of the flow of materials between various locations (farms, abattoirs, retailers, and customers) and not on considering specific production constraints relevant to our research. However, the author considers the perishability, shelf life, and freshness of the products. The paper also contains a broad review of supply chain works for various production types, showing that meat-related (or even food-related) supply chain papers are a rarity.

\subsection{Models considering MOQ Constraint}
As we have seen from the above papers, MOQ is rarely considered in meat production optimization. However, such a constraint is used in several papers addressing general supply chains and lot-sizing. We will very briefly mention several such papers here. For the supply chain problems, \cite{das2003modeling} use MOQ as one of the parameters they considered in their research concerning the flexibility of supply chains. A similar paper by \cite{kesen2010evaluating} considered supply chain flexibility in the presence of order quantity constraint, including MOQ. Regarding lot-sizing problems, \cite{lee2004inventory} considered a dynamic lot-sizing with MOQ and presented a polynomial algorithm for the problem. \cite{hellion2012polynomial} consider MOQ as one of the constraints in their single-item capacitated lot-sizing problem and provide a polynomial time algorithm. A similar paper by \cite{li2016polynomial} considered polynomial solvability and optimality conditions for several lot-sizing problem variants, including MOQ. The authors also included several examples for each case.

\subsection{Summary}
To summarize, the problem of optimizing the production process focused on meat processing has seen limited attention in the literature. While some approaches exist (four papers), they do not account for many specific production properties encountered in practice. Moreover, existing papers often focus on supply chains and do not consider cutting and meat processing in detail. Most importantly, only one paper considered the MOQ constraint in the context of the meat supply chain, though it did not focus on meat processing details. Similarly, only one paper explicitly considered multiple recipes for a single final product, but even so, this concept significantly differs from our concept of alternatives. We conclude that, to the best of our knowledge, this paper is the first to consider both the MOQ constraint and product alternatives while also focusing on the details of the production process, e.g., the MPA constraint. In Table~\ref{tab:literature}, we show the comparison between our paper and existing approaches. A smaller-size checkmark is used when \textcolor{black}{the aspect is not fully covered or the situation is not clear} (e.g., the concept is not named or defined explicitly but seems to be implied or partially similar to ours).

\begin{table*}
\centering
\footnotesize
\caption{Comparison of the relevance of existing literature approaches to our work}
\label{tab:literature}
\begin{tabular}{cccccccc}
    \toprule
     Paper & \rotatebox{90}{Meat} & \rotatebox{90}{Cutting}\rotatebox{90}{patterns} & \rotatebox{90}{Recipes} & \rotatebox{90}{MOQ} & \rotatebox{90}{Groups of}\rotatebox{90}{alternatives} & \rotatebox{90}{Expiration}\rotatebox{90}{date} & \rotatebox{90}{Chilled and}\rotatebox{90}{and frozen}\rotatebox{90}{material}\rotatebox{90}{on stock} \\
    \midrule
    \cite{das2003modeling} & & & & {\large\checkmark} & & & \\
    \cite{lee2004inventory} & & & & {\large\checkmark} & & & \\
    \cite{cooperative2005supply} & {\large\checkmark} & {\large\checkmark} & {\tiny\checkmark} & & {\tiny\checkmark} & & \\
    \cite{wikborg2008online} & {\large\checkmark} & {\large\checkmark} & {\tiny\checkmark} & & & & \\
    \cite{schutz2009supply} & {\large\checkmark} & & {\large\checkmark} & & & & \\
    \cite{kesen2010evaluating} & & & & {\large\checkmark} & & & \\
    \cite{reynisdottir2012linear} & {\large\checkmark} & {\large\checkmark} & {\tiny\checkmark} & & & & \\
    \cite{hellion2012polynomial} & & & & & & & \\
    \cite{albornoz2015optimization} & {\large\checkmark} & {\large\checkmark} & {\tiny\checkmark} & & & {\tiny\checkmark} & {\tiny\checkmark} \\
    \cite{rijpkema2016application} & {\large\checkmark} & & & & & & \\
    \cite{li2016polynomial} & & & & {\large\checkmark} & & & \\
    \cite{mohammed2017fuzzy} & {\large\checkmark} & & & & & & \\
    \cite{mohammed2017multi} & {\large\checkmark} & & & & & & \\
    \cite{Arash2017fish} & {\large\checkmark} & & & {\large\checkmark} & & & \\
    \cite{rodriguezcoordination} & {\large\checkmark} & {\large\checkmark} & {\tiny\checkmark} & & & & \\
    \cite{leksakul2019fuzzifying} & {\large\checkmark} & {\large\checkmark} & {\tiny\checkmark} & & & & \\
    \cite{mohebalizadehgashti2019multi} & {\large\checkmark} & & & & & & \\
    \cite{blanco2020} & {\large\checkmark} & & & {\large\checkmark} & & & \\
    \cite{Brevik2020broiler} & {\large\checkmark} & & & {\large\checkmark} & & & \\
    \cite{doi:10.1080/01605682.2019.1640588} & {\large\checkmark} & & & & & & \\
    \cite{https://doi.org/10.1111/itor.13014} & {\large\checkmark} & & & & & & \\
    \cite{schmidt2022traceability} & {\large\checkmark} & & & & & & \\
    \cite{koroteev2022optimization} & {\large\checkmark} & {\tiny\checkmark} & {\tiny\checkmark} & & & {\tiny\checkmark} & \\
    \cite{arabsheybani2024sustainable} & {\large\checkmark} & & & & & {\large\checkmark} & {\tiny\checkmark} \\
    \cite{FATHOLLAHZADEH2024110578} & {\large\checkmark} & & & & & & {\tiny\checkmark} \\
    This work & {\large\checkmark} & {\large\checkmark} & {\large\checkmark} & {\large\checkmark} & {\large\checkmark} & {\large\checkmark} & {\large\checkmark} \\  
     \bottomrule
\end{tabular}
\end{table*}

\section{Problem Description}\label{sec:problemstatement}

The problem is given by a set of materials and recipes. Materials include all input materials, intermediary components, and final products. Regarding recipes, it should first be noted that a production process can be divided into two stages: cutting and production. In the cutting stage, a~single input item (e.g., a~pig carcass) usually yields multiple output items (side, leg, belly, etc.). Recipes in this stage are often called \textit{cutting patterns} in the literature\textcolor{black}{, and a similar structure of recipes can be observed also in the dairy industry}. The recipes in the production stage work mostly in the opposite direction: multiple input items are used to make a~single output item (e.g., a~specific type of sausage). Nevertheless, the recipes in the problem at hand are more complex in that they can transform multiple materials into multiple other materials. Therefore, in our formulation, we will model the recipes both in the cutting and production stages in the same way. As such, a~recipe specifies how the materials change in the meat processing, regardless of whether it is a cutting process or production. Hence, a recipe defines a set of input materials together with their quantities and a set of output materials with their quantities. In addition, \textcolor{black}{the meat processing company also defines} groups of alternative materials in a recipe, meaning that the materials in the same group of alternatives are mutually interchangeable but need to sum up to the required quantity. Note that a recipe may also be called a bill of materials or a piece list.

An illustrative instance with three recipes and six materials can be seen in Figure~\ref{fig:example}. The circles are materials, the squares are recipes. The arrows depict the flows of materials and their quantities specified by recipes. For example, recipe 1 specifies that from 1000 units of material 1, using this recipe, one can produce 50 units of material 2 and 900 units of material 3. There is also one group of alternatives, depicted by dashed lines. Hence, recipe 3 is defined by yielding material 6 in the quantity of 1400, for which it requires material 3 in the quantity of 1200 and an arbitrary mix of materials 4 and 5 in the total quantity of 300. Next, notice that material 3 is output from recipes 1 and 2. Hence, a material may be produced by multiple recipes. \textcolor{black}{The key in the bottom-right corner of the figure illustrated notation related to the PPOP that will be introduced in the next section.}



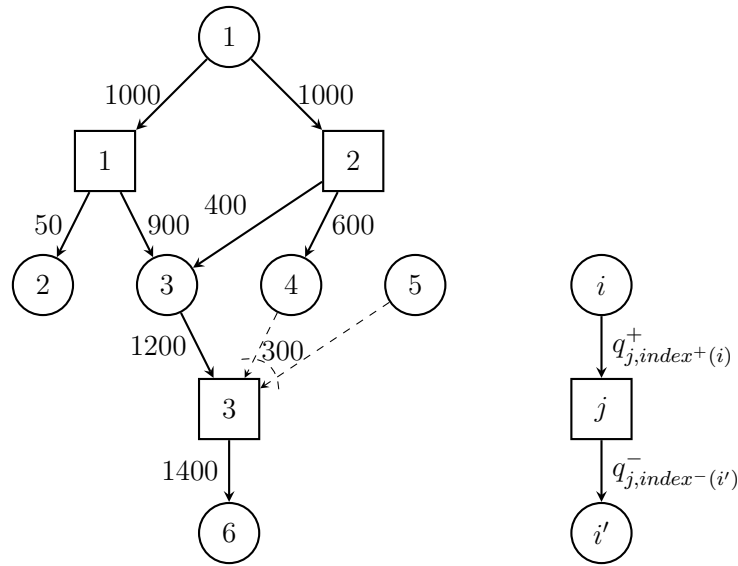
\begin{figure*}[tp]
\centering
\newcommand{\mat}[3] { 
 	\node[circle, draw, thick, minimum size=0.8cm] (m#1) at (#2,#3) {$#1$};
}
\newcommand{\matcs}[4] { 
 	\node[circle, draw, thick, minimum size=0.8cm] (mc#1) at (#2,#3) {$\widetilde{#1}$};
 	\node[font=\scriptsize] (d#1) at (#2+0.5,#3-0.75) {$h_{\widetilde{#1}}=#4$};
}
\newcommand{\kus}[3] { 
 	\node[draw, thick, minimum size=0.8cm] (k#1) at (#2,#3) {$#1$};
}

\begin{tikzpicture}[scale=0.55, draw=black, text=black]
\mat{1}{0}{0}
\kus{1}{-3}{-3}
\kus{2}{3}{-3}
\mat{2}{-4.5}{-6}
\mat{3}{-1.5}{-6}
\mat{4}{1.5}{-6}
\mat{5}{4.5}{-6}
\kus{3}{0}{-9}
\mat{6}{0}{-12}

\draw[-stealth, thick] (m1) -- node[anchor=east] {$1000$} (k1);
\draw[-stealth, thick] (m1) -- node[anchor=west] {$1000$} (k2);

\draw[-stealth, thick] (m3) -- node[anchor=east] {$1200$} (k3);
\draw[-stealth, dashed] (m4) -- (k3);
\draw[-stealth, dashed] (m5) -- (k3);

\draw[-stealth, thick] (k1) -- node[anchor=east] {$50$} (m2);
\draw[-stealth, thick] (k1) -- node[anchor=west] {$900$} (m3);
\draw[-stealth, thick] (k2) -- node[anchor=south east] {$400$~} (m3);
\draw[-stealth, thick] (k2) -- node[anchor=west] {$600$} (m4);

\draw[-stealth, thick] (k3) -- node[anchor=east] {$1400$} (m6);

\node[] (f) at (1.3,-7.6) {$300$};
\draw[dashed] (0.3,-7.8) to [controls=+(90:0) and +(90:1)] (1.2,-8.6);

\mat{i}{9}{-6}
\kus{j}{9}{-9}
\mat{i'}{9}{-12}
\draw[-stealth, thick] (mi) -- node[anchor=west] {$q^+_{j,index^+(i)}$} (kj);
\draw[-stealth, thick] (kj) -- node[anchor=west] {$q^-_{j,index^-(i')}$} (mi');

\end{tikzpicture}

\caption{Illustrative instance of the purchase and production optimization problem.}
\label{fig:example}
\end{figure*}

\textcolor{black}{The goal is to satisfy a~given demand for individual materials using recipes existing in the production system.} The materials can be acquired by buying them, taking them from stock, or producing them using recipes, and oppositely, materials are used to satisfy demand, to be put in stock, or to be processed further in production using recipes.
Each material is specified by its cost (price), demand, shelf life (expiration), turnover (quantity used or sold per given time frame), and quantity we already have in stock. Besides, bulks of the same material may be in stock with different shelf lives.
\textcolor{black}{In addition, the model assumed in this paper allows us to decide whether the material will be chilled or frozen when it goes to stock. This is defined via specific recipes allowing the conversion of a specific material to a frozen one. Another set of recipes is introduced for the reverse process. This allows us to define not only which material can or cannot be put into stock in which form but also defines the change of the material in terms of cost and weight (e.g., meat loses weight when it is frozen and unfrozen).}
Note that we define demands not only for final products but for all materials since the company also sells various intermediary components \textcolor{black}{including their fresh, chilled, and frozen variants}.

There are two extra requirements specified by the meat processing company. First, the MOQ requirement specifies that if a material is going to be bought, there is given a minimum quantity to be bought. MOQ can be different for each material.
Second, \textcolor{black}{by the MPA requirement the company specifies} that when some material from a group of alternatives is to be used, then the ratio of the quantity of this material to the total quantity of the materials from that group of alternatives must be at least 5\%. In the illustrative instance from Figure~\ref{fig:example}, the materials 4 and 5 must not be mixed in the ratio, say, 6 units of material 4 and 294 units of material 5 as the percentage of material 4 in the total quantity in the group of alternatives would be only 2\%.

The aim is to decide how much each recipe will be applied, how much of each material will be bought, and how much will remain to be put in stock while minimizing a compound objective function \textcolor{black}{assuming several aspects. The meat processing company defined five relevant objectives intending to have a choice of which one(s) would be used. The motivation is that the changing situation in the market requires adjustment of the criterion function, and the company also needs to analyze different buying strategies.} First of all, they want to minimize the total nominal cost of materials to be bought. Next, they want to minimize the nominal cost of materials that remain to be put in stock, with the intention of reducing inventory holding costs and maximizing freshness primarily of the more valuable materials. Furthermore, the materials with shorter shelf life are preferred not to be overproduced, and the materials that are already stored are preferred to be processed (from the oldest to the newest). Finally, materials with low stock turnover are preferred not to be overproduced. Note that we talk about overproduction as it is hardly possible to meet all the demands such that nothing is left to be put in stock. \textcolor{black}{The first two objectives are assumed to be the most important ones, and since they both represent a direct cost, they can be seen as a single objective. The last three are seen as auxiliary and give the company an option to modify its buying strategy. Therefore, the problem is seen as a single objective optimization problem, while the company can vary its form using a weighted linear function. A closer discussion regarding the specific choices of weights is discussed in Section~\ref{sec:weights}.}

\section{Mathematical Model}\label{sec:matmodeldesign}

In this section, we first describe the mathematical programming model to solve the PPOP. The first subsection describes the parameters and constraints; subsequently, we discuss the objective function. The third subsection describes requirements and related constraints that make the problem \mbox{$\mathcal{NP}$-hard}. The last subsection describes the approach that mitigates the numerical instability of the proposed ILP model.
The main notation is summarized in Table~\ref{tab:notation}. All numerical parameters of the problem and all introduced variables are non-negative real-valued unless explicitly stated differently.

\subsection{Parameters and Constraints}

The problem is given by a set of materials $M$ and a set of recipes $J$. For each material $i \in M$, we are given demand $d_i$, cost $c_i$, and the quantity we already have in stock $h_i$. Let us first define variables $b_i$ to represent the quantity of material $i \in M$ to buy, $s_i$ to represent the quantity of material $i \in M$ to be put in stock, and $z_j$ to represent how much to cut/produce by recipe $j \in J$. The following paragraphs describe individual types of constraints.

\begin{table*}[t]
\scriptsize
    \caption{Parameters and variables}
    \label{tab:notation}
    \centering
    \begin{tabular}{l|l}
    \toprule
        \textbf{Symbol} & \textbf{Meaning} \\
        \midrule 
\multicolumn{2}{c}{\bf{}Parameters} \\
\midrule 
$M$ & set of materials \\
$J$ & set of recipes \\
$d_i$ & demand for product $i \in M$ \\
$c_i$ & cost of material $i \in M$ \\
$h_i$ & how much material $i \in M$ we initially hold in stock \\
$b^{min}_i$ & MOQ of material $i \in M$ \\
$trn_i$ & stock turnover of material $i \in M$, specified by quantity used per period of time \\
$exp_i$ & shelf life of freshly produced material $i \in M$ \\
$m_{i,1}, \dots, m_{i,k_i}$ & 
list of quantities (batches) of material $i$ held in stock \\
$exp_{i,1}, \dots, exp_{i,k_i}$ &
list of remaining shelf lives of batches of material $i$ held in stock \\
$M^+_j$ & input materials of recipe $j$ \\
$Q^+_j$ & quantities of input materials of recipe $j$ \\
$M^-_j$ & output materials of recipe $j$ \\
$Q^-_j$ & quantities of output materials of recipe $j$ \\
$Alts$ & set of tuples $(j,A,q_{j,A})$, where $j$ is a recipe, $A$ is a group of alternatives, and $q_{j,A}$ is the quantity \\

        \midrule 
\multicolumn{2}{c}{\bf{}Decision variables} \\
\midrule 
$z_j$ & how much to cut/produce by recipe $j \in J$ \\
$\hat{z}_{j,A,i}$ & proportion of using material $i$ within group of alternatives $A$ in recipe $j$ \\
$b_i$ & quantity of material $i \in M$ to buy \\
$s_i^{new}$ & quantity of material $i \in M$ to be put in stock from newly produced \\
$s_i^{old}$ & quantity of material $i \in M$ to be put in stock from what was already in stock \\
        \midrule 
\multicolumn{2}{c}{\bf{}Auxiliary variables} \\
\midrule 
$s_i$ & quantity of material $i \in M$ to be put in stock in total \\
$p_i$ & production of material $i \in M$ in recipes \\
$u_i$ & usage of material $i \in M$ in recipes \\

\bottomrule
    \end{tabular}
\end{table*}

\subsubsection*{Material Flow Constraint}

Let us first define each recipe $j \in J$ by four lists: (i) input materials $M^+_j$, (ii) quantities of input materials $Q^+_j = \left( q^+_{j,1}, q^+_{j,2}, \dots, q^+_{j,{|M^+_j|}}\right)$, (iii) output materials $M^-_j$, and (iv) quantities of output materials $Q^-_j = \left( q^-_{j,1}, q^-_{j,2}, \dots, q^-_{j,{|M^-_j|}}\right)$.
Suppose $index^+(i)$ and $index^-(i)$ give the index of $i$ in list $M^+_j$ and $M^-_j$, respectively.

For each material $i \in M$, let the quantity of the material produced by using/applying the recipes be defined as follows:
\begin{equation}\label{constrProduction}
p_i = \sum_{j \in J | i \in M^-_j} z_j \cdot q^-_{j,index^-(i)} \end{equation}

The sum goes over all recipes that contain material $i$ on the output. Analogically, the usage of material $i$ for further processing, i.e., the quantity of material $i$ consumed/used by the recipes, could be defined as:
\begin{equation}\label{constrUsage}
u_i = \sum_{j \in J | i \in M^+_j} z_j \cdot q^+_{j,index^+(i)} \end{equation}

For example, in the instance from Figure~\ref{fig:example}, the usage of material 3 in recipes is $u_3 = 1200z_3$, and the production of material 3 in recipes is $p_3 = 900z_1 + 400z_2$.

For each material $i \in M$, the output and the input of the material must be equal. More precisely, it must hold that the quantity which goes to satisfy demand, plus the quantity to be put in stock, plus usage, must be equal to the quantity to be bought, plus the quantity from stock, plus the quantity produced by using the recipes. Hence, $\forall i \in M$:
\begin{equation}
d_i + s_i + u_i = b_i + h_i + p_i
\end{equation}

\subsubsection*{Alternatives}

Let us define $Alts$ as a set of tuples $(j,A,q_{j,A})$, where $j$ is a recipe, $A$ is a group of alternatives, i.e., a list of alternative (mutually interchangeable) materials, and $q_{j,A}$ is the total quantity of materials from group of alternatives $A$ defined in recipe $j$. For each tuple $(j,A,q_{j,A}) \in Alts$ and for each material $i \in A$, we need to introduce variable $\hat{z}_{j,A,i}$ that represents the proportion of using material $i$ within group of alternatives $A$ in recipe $j$. Then, we add the following constraint, $\forall (j,A,q_{j,A}) \in Alts$:
\begin{equation}\label{con:zalt}
z_j = \sum_{i\in A} \hat{z}_{j,A,i}
\end{equation}

Next, we extend equation~(\ref{constrUsage}), i.e., the usage of material $i$ in recipes, by adding a new expression for all occurrences of material $i$ in groups of alternatives:
\begin{equation}
u_i = \sum_{j \in J | i \in M^+_j} z_j \cdot q^+_{j,index^+(i)} + \sum_{(j,A,q_{j,A}) \in Alts | i \in A} \hat{z}_{j,A,i} \cdot q_{j,A} \end{equation}

Note that there may be recipes where one material can be present as required input as well as in the group of alternatives. That is, there may be recipe $j \in J$ such that material $i \in M$ occurs both in $M^+_j$ and $A$ for some $(j,A,q_{j,A})$ (even in multiple groups of alternatives within the same recipe). 

The model related to the instance from Figure~\ref{fig:example} defines variables $z_1, z_2, z_3$; one for each recipe. Recipe 3, for example, is defined as $M^+_j = \left(3\right)$, $Q^+_j = \left(1200\right)$, $M^-_j = \left(6\right)$, and $Q^-_j = \left(1400\right)$. There is one group of alternatives $A = \{4,5\}$, $q_{3,A} = 300$, and $Alts = \{(3, A, 300)\}$. There will also be variables $\hat{z}_{3,A,4}$ and $\hat{z}_{3,A,5}$, and constraint~(\ref{con:zalt}) yields equation $z_3 = \hat{z}_{3,A,4} + \hat{z}_{3,A,5}$.
The values can be set, for example, such that $z_3 = 1/2$, $\hat{z}_{3,A,4} = 1/3$, and $\hat{z}_{3,A,5} = 1/6$. In this case, material 3 would go to recipe 3 in the quantity of 600, material 6 would go from recipe 3 in the quantity of 700, and materials 4 and 5 would go to recipe 3 in the quantities of 100 and 50, respectively.

The company defines that a material can be bought only if it cannot be produced by any recipe, i.e., $\forall i \in M$, if $\{j \in J | i \in M^-_j\} \not= \emptyset$, then $b_i = 0$.
In the instance from Figure~\ref{fig:example}, only materials 1 and 5 can be bought.
When a material can be bought and produced, then it is represented as two different materials in set $M$.
In the same way, whether a bought or produced material is frozen or subsequently defrosted is distinguished by representing them as different materials.
Then a transition from a material to its frozen state is modeled using a freezing recipe and the transition from a frozen material to its defrosted state by an unfreezing recipe.
Whenever some of these materials are mutually interchangeable in a recipe (e.g., bought, produced, or defrosted), it is modeled as a group of alternatives.

\subsubsection*{Stock}

Since we aim to consider inventory costs as well as the perishability of materials, we need to distinguish what is to be put in stock from newly produced material and what is from the already stored material.
We assume that there may be multiple batches of each material in stock, differing in the expiration date. \textcolor{black}{The expiration date depends on the date when the material was produced, its type, and whether it is kept chilled or frozen. In addition, some materials can be transformed into frozen ones and back using recipes defined by the company.} Different expiration dates of the batches in stock are represented using the objective function described in the next subsection. Therefore, for now, we only distinguish between the newly produced material and the one already in stock.

We introduce variables, for each $i \in M$, $s_i^{new}$ to represent what is left from the newly produced material and $s_i^{old}$ to represent what is left from what we already had in stock. To ease the description, we also introduce variables $s_i$ to represent what is left in total, and hence, we add this equation for each $i \in M$:
\begin{equation}
s_i = s_i^{new} + s_i^{old}
\end{equation}

To make $s_i^{old}$ correctly reflect the remainder from what was already in stock (which is $h_i$), first, $s_i^{old}$ cannot be more than $h_i$, i.e., we add the following inequality, $\forall i \in M$:
\begin{equation}
s_i^{old} \le h_i
\end{equation}

Next, $s_i^{old}$ cannot be less than what was already in stock ($h_i$) minus what was applied on satisfying the demand ($d_i$) and used by recipes ($u_i$), i.e., we add the following inequality, $\forall i \in M$:
\begin{equation}
s_i^{old} \ge h_i - d_i - u_i
\end{equation}

Variable $s_i^{old}$ will be pushed to its minimum possible value (hence, the stock material will be primarily used) by the objective function, which will be described next. Hence, if $d_i + u_i \ge h_i$, then $s_i^{old} = 0$, i.e., all material from stock will be depleted. If $d_i + u_i < h_i$, then $s_i^{old} = h_i - d_i - u_i$.

\textcolor{black}{Batches of material $i$ with different expiration are modeled using the objective function, and this will be described in the next section.}

\subsection{Objective Function}

\textcolor{black}{The meat processing company defined five different aspects of optimization and requested to assume it as a compound objective with the possibility to adjust the emphasis on individual components by weights, e.g., using a set of predefined weights with a possibility to ignore selected aspects (for more details see Section~\ref{sec:weights}). Hence, the optimization objective we consider is a weighted sum of five functions:}
\begin{equation}
\min w_0 \cdot f_0 + w_1 \cdot f_1 + w_2 \cdot f_2 + w_3 \cdot f_3 + w_4 \cdot f_4,
\end{equation}
where $w_0, w_1, ..., w_4$ are the weights associated with the functions. The individual functions are the following:
\begin{description}
    \item[$f_0$] nominal cost of materials to be bought, i.e., $\sum_{i \in M} c_i \cdot b_i$
    \item[$f_1$] nominal cost of materials to be put in stock, i.e., $\sum_{i \in M} c_i \cdot s_i$
    \item[$f_2$] the materials of low stock turnover are preferred not to remain to be put in stock, which can be formalized as minimizing $\sum_{i \in M} e^{-trn_i} \cdot s_i$, where $trn_i$ denotes the stock turnover of material $i \in M$, specified by the quantity used per period of time
    \item[$f_3$] the materials of short shelf life are preferred not to be overproduced, which can be formalized as minimizing $\sum_{i \in M} e^{-exp_i} \cdot s_i^{new}$, where $exp_i$ denotes the shelf life of newly produced material $i \in M$
    \item[$f_4$] from the materials already stored, the materials of upcoming expiration are used primarily, which is elaborated further 
\end{description}


There may be multiple batches of each material in stock, each with a different expiration date. Therefore, we are also given, for each stored material $i$, a list of quantities $m_{i,1}, \dots, m_{i,k_i}$ and a list of remaining shelf lives $exp_{i,1}, \dots, exp_{i,k_i}$.
We need $f_4$ to incur a penalty according to the expiration such that the oldest batch is taken first and the newest batch last. Let us first sort and reindex the batches such that $exp_{i,1} > exp_{i,2} > \dots > exp_{i,k_i}$, i.e., from the newest to the oldest one.

An illustration of the dependence of the penalty on the quantity of material $i$ to be put in stock again is depicted in Figure~\ref{pwl}.
It can be immediately seen that it leads to the concept of a piecewise linear function (PWL), which is piecewise increasing on the interval from 0 to $h_i$.

\begin{figure*}[t]
\centering
\begin{tikzpicture}[scale=1, draw=black, text=black]
\begin{axis}[
  axis x line=bottom,
  axis y line=left,
    width=10cm,
    height=6cm,
    	xlabel={$s_i^{old}$},
    	ylabel={$PWL_i(s_i^{old})$},
     xmajorgrids={true},
     ymajorgrids={true},
	major grid style={dashed, thick},
	tick align=outside,
	ymin=0,
	ymax=3.15,
	xmin=0,
	xmax=4.15,
	ytick={0,0.2,0.65,1.7,3},
	xtick={0,1,2,3.2,4},
	yticklabels={~,$m_{i,1}\cdot e^{-exp_{i,1}}$,$m_{i,1}\cdot e^{-exp_{i,1}} + m_{i,2}\cdot e^{-exp_{i,2}}$,$\sum_{k' = 1}^{3} m_{i,k'}\cdot e^{-exp_{i,k'}}$,$\sum_{k' = 1}^{4} m_{i,k'}\cdot e^{-exp_{i,k'}}$},
	xticklabels={~,$m_{i,1}$,$m_{i,1}+m_{i,2}$,$\sum_{k' = 1}^{3} m_{i,k'}$,$h_i$},
	x tick label style={font=\scriptsize},
	y tick label style={font=\scriptsize}]
	
     \draw plot coordinates {(0,0) (1,0.2) (2,0.65) (3.2,1.7) (4,3)};
\end{axis}
\end{tikzpicture}
\caption{Piece-wise linear function, an illustration for four segments.}
\label{pwl}
\end{figure*}
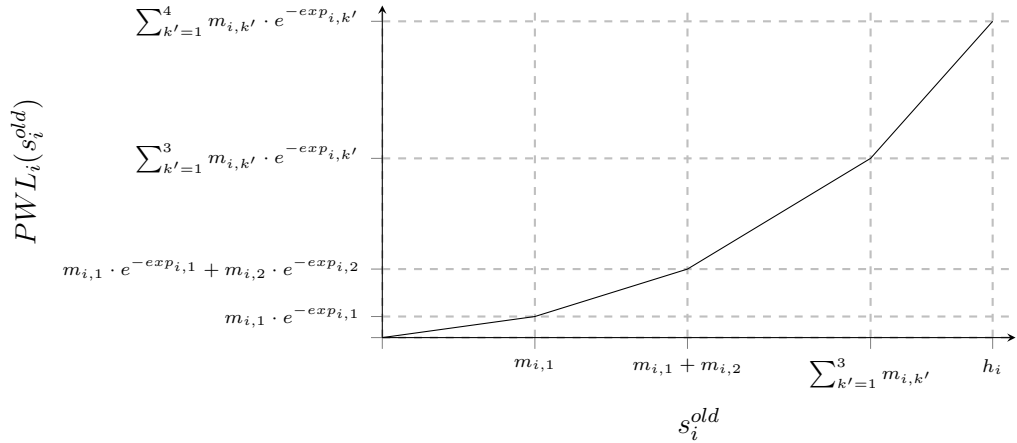

The PWL function is constructed as follows. If a solver provides a modeling interface for PWL, we need a list of x- and y-coordinates where the slope of the function changes. We start with the origin, i.e., $x_{i,0} = 0$ and $y_{i,0} = 0$, and the other coordinates, for each $k \in \{1, \dots, k_i\}$, are computed as follows:
\begin{equation}\label{prusx}
x_{i,k} = \sum_{k' = 1}^{k} m_{i,k'}
\end{equation}
\begin{equation}\label{prusy}
y_{i,k} = \sum_{k' = 1}^{k} m_{i,k'}\cdot e^{-exp_{i,k'}}
\end{equation}

With such PWL defined, the overall optimization objective is as follows:
\begin{equation}
\min \sum_{i \in M} w_0 \cdot c_i \cdot b_i + w_1 \cdot c_i \cdot s_i + w_2 \cdot e^{-trn_i} \cdot s_i 
+ w_3 \cdot e^{-exp_i} \cdot s_i^{new} + w_4 \cdot PWL_i(s_i^{old})
\label{eq:PWLObjective}
\end{equation}

If the solver does not provide the modeling interface for PWL, we transform it to LP as follows. We need to express each part of the function in the form of $ax + b$ and then add constraint $ax + b \le t$, where $t$ is an auxiliary variable that is added to the minimization objective. It can be seen that for a fixed $x$, it holds in an optimal solution that $t = ax + b$.

In our case, for material $i$ and $k$-th segment of the PWL function, the slope of the function is $e^{-exp_{i,k}}$. Variable $x$ represents the remaining material from the already stored in stock, which is $s_i^{old}$. Hence, the function can be expressed as:
\begin{equation}
e^{-exp_{i,k}} \cdot s_i^{old} + \overline{y}_{i,k}
\end{equation}
where
\begin{equation}\label{prusecik}
\overline{y}_{i,k} = y_{i,k} - e^{-exp_{i,k}} \cdot x_{i,k}
\end{equation}

We then need to add to our model the following constraint, $\forall i \in M, \forall k \in \{1, \dots, k_i\}$:
\begin{equation}\label{pwlelim}
e^{-exp_{i,k}} \cdot s_i^{old} + \overline{y}_{i,k} \le r_i
\end{equation}
where $\overline{y}_{i,k}$ is calculated from equation (\ref{prusecik}) and $r_i \in \mathbb{R}^+_0$ is the auxiliary variable that we add to the objective function. Applying (\ref{prusx}), (\ref{prusy}), and (\ref{prusecik}) to constraint (\ref{pwlelim}), we get, $\forall i \in M, \forall k \in \{1, \dots, k_i\}$:
\begin{equation}
e^{-exp_{i,k}} \cdot s_i^{old} - r_i \le e^{-exp_{i,k}} \cdot \sum_{k' = 1}^{k} m_{i,k'} - \sum_{k' = 1}^{k} m_{i,k'}\cdot e^{-exp_{i,k'}}
\end{equation}

Hence, the optimization objective is:
\begin{equation}
\min \sum_{i \in M} w_0 \cdot c_i \cdot b_i + w_1 \cdot c_i \cdot s_i + w_2 \cdot e^{-trn_i} \cdot s_i + w_3 \cdot e^{-exp_i} \cdot s_i^{new} + w_4 \cdot r_i
\end{equation}

Since the exponential function plummets very rapidly, we divide the parameters $exp_i$ and $trn_i$ by 5000 to keep the numbers in the model in reasonable ranges. The number 5000 was found to be suitable by empirical observation. The choice of weights $w_0, \dots, w_4$ will be explained in Section~\ref{sec:weights}.

\subsection{Extra Requirements}

It remains to resolve two extra requirements, due to which the model is no longer on continuous variables. The first requirement is MOQ. We are given, $\forall i \in M$, the minimum required quantity $b^{min}_i$ of product $i$ stating that the quantity of material $i$ to be bought is either 0 or at least $b^{min}_i$. In other words, if we decide to buy product $i$, i.e., $b_i > 0$, then we must buy it in the quantity of at least $b^{min}_i$, i.e., $b_i \ge b^{min}_i$. To achieve this, we add an auxiliary binary variable $v_i$ and the following constraint, $\forall i \in M$:

\begin{align}\label{constrMinBuy}
&b_i \le 0 + v_i \cdot \mathcal{M}\\
&b_i \ge b^{min}_i - (1-v_i)\cdot \mathcal{M}\\
&v_i \in \{0, 1\}
\end{align}

where $\mathcal{M}$ is some sufficiently large constant, thanks to which one of the constraints is always 'switched off' according to the value of $v_i$.

The next requirement is MPA, which states that each material in a group of alternatives must participate by at least 5 \%, or not at all, i.e., the ratio $\hat{z}_{j,A,i}$ to $z_j$ must not be between 0 and 0.05. Hence, $\forall (j,A,q_{j,A}) \in Alts, i \in A$:

\begin{align}\label{constrMinProb}
&\hat{z}_{j,A,i} \le 0 + \hat{v}_{j,i} \cdot \mathcal{M}\\
&\hat{z}_{j,A,i} \ge 0.05 \cdot z_j - (1-\hat{v}_{j,i}) \cdot \mathcal{M}\\
&\hat{v}_{j,i} \in \{0, 1\}
\end{align}

Notice the difference that MOQ defines absolute minimum as $b^{min}_i$ is a constant parameter, whereas MPA defines relative minimum in a group as $z_j$ is a variable.


When we summarize the variables and constraints introduced above, the complete model looks as follows: 

\vspace{0.8em}
\noindent
\textbf{Minimize:}
\begin{equation}
\sum_{i \in M} w_0 \cdot c_i \cdot b_i + w_1 \cdot c_i \cdot s_i + w_2 \cdot e^{-trn_i} \cdot s_i + w_3 \cdot e^{-exp_i} \cdot s_i^{new} + w_4 \cdot r_i \label{con:lpstart}
\end{equation}
\textbf{Subject to:}
\begin{align}
d_i + s_i + u_i = b_i + h_i + p_i \quad 
& \forall i \in M\\
p_i = \sum_{j \in J | i \in M^-_j} z_j \cdot q^-_{j,index^-(i)} \quad
& \forall i \in M\\
u_i = \sum_{j \in J | i \in M^+_j} z_j \cdot q^+_{j,index^+(i)} + \sum_{(j,A,q_{j,A}) \in Alts | i \in A} \hat{z}_{j,A,i} \cdot q_{j,A} \quad
& \forall i \in M\\
z_j = \sum_{i\in A} \hat{z}_{j,A,i} \quad & \forall (j,A,q_{j,A}) \in Alts\\
s_i = s_i^{new} + s_i^{old} \quad
& \forall i \in M\\
h_i - d_i - u_i \le s_i^{old} \quad
& \forall i \in M\\
s_i^{old} \le h_i \quad
& \forall i \in M\\
e^{-exp_{i,k}} \cdot s_i^{old} - r_i \le e^{-exp_{i,k}} \cdot \sum_{k' = 1}^{k} m_{i,k'} - \sum_{k' = 1}^{k} m_{i,k'}\cdot e^{-exp_{i,k'}} \quad \label{con:lpend}
& \forall i \in M, \forall k \in \{1, \dots, k_i\}\\
b_i \le 0 + v_i \cdot \mathcal{M} \quad \label{con:furbmina}
& \forall i \in M\\
b_i \ge b^{min}_i - (1-v_i)\cdot \mathcal{M} \quad \label{con:furbminb}
& \forall i \in M\\
\hat{z}_{j,A,i} \le 0 + \hat{v}_{j,i} \cdot \mathcal{M} \quad \label{con:furpsta}
& \forall (j,A,q_{j,A}) \in Alts, i \in A\\
\hat{z}_{j,A,i} \ge 0.05 \cdot z_j - (1-\hat{v}_{j,i}) \cdot \mathcal{M} \quad \label{con:furpstb}
& \forall (j,A,q_{j,A}) \in Alts, i \in A\\
b_i, s_i, s_i^{new}, s_i^{old}, p_i, u_i, r_i \in \mathbb{R}^+_0 \quad \label{con:lpvara}
& \forall i \in M\\
z_j \in \mathbb{R}^+_0 \quad
& \forall j \in J\\
\hat{z}_{j,A,i} \in \mathbb{R}^+_0 \quad \label{con:lpvarb}
& \forall (j,A,q_{j,A}) \in Alts, i \in A\\
v_i \in \{0, 1\} \quad \label{con:furbminc}
& \forall i \in M\\
\hat{v}_{j,i} \in \{0, 1\} \quad \label{con:furpstc}
& \forall (j,A,q_{j,A}) \in Alts, i \in A
\end{align}

Variables $s_i$, $u_i$, and $p_i$ do not need to be used and serve solely for clarity. 

\subsection{Algorithmic Design}
\label{sec:algorithm}

The disjunctive constraints realized by binary variables and big-$\mathcal{M}$ constants make the model more computationally difficult to solve.
Besides, due to the large ranges of values in the data, the model \textcolor{black}{often leads to numerical issues on data from real production.}
For these reasons, we propose the following solution approach based on iterative constraint generation \textcolor{black}{that is safer regarding numerical issues and allows solving real-world instances using a non-commercial ILP solver.}

We initially \textcolor{black}{consider a~subproblem of PPOP by taking into account only constraints} (\ref{con:lpstart})--(\ref{con:lpend}) and (\ref{con:lpvara})--(\ref{con:lpvarb}). \textcolor{black}{Since we have only linear constraints, we construct a~LP model for this subproblem and invoke the solver. Since this is a~LP formulation, we get an~optimal solution} \textcolor{black}{in polynomial time.}

After finding the optimal solution to the LP model \textcolor{black}{for the subproblem, we check whether the solution is feasible for the original full problem by checking the constraints not covered by the LP formulation. Specifically,} we check whether the extra requirements MOQ and MPA are satisfied. That is, we verify that for all materials, their quantity to be bought is either 0 or at least $b^{min}_i$. If not, we add constraints (\ref{con:furbmina}), (\ref{con:furbminb}), and (\ref{con:furbminc}), $\forall i \in M ~|~ b_i \in (0,b^{min}_i)$. 
Furthermore, if for some group of alternatives $(j,A,q_{j,A})$ and an alternative material $i \in A$, it holds that the ratio $\hat{z}_{j,A,i}$ to $z_j$ is between 0 and 5 percent, then we add constraints (\ref{con:furpsta}), (\ref{con:furpstb}), and (\ref{con:furpstc}), $\forall (j,A,q_{j,A}) \in Alts, i \in A ~|~ \frac{\hat{z}_{j,A,i}}{z_j} \in (0,0.05)$.

After this construction, we invoke the solver again to find an optimal solution to the new ILP model. After getting the optimal solution, we again verify the extra requirements MOQ and MPA, and this process is iteratively repeated until we get a solution that satisfies all the \textcolor{black}{MOQ and MPA} requirements.

We will refer to this approach as the \textit{iterative model}, while the complete ILP model containing all the constraints (\ref{con:lpstart})--(\ref{con:furpstc}) from the beginning will be referred to as the \textit{global model}. A comparison of these two approaches to solving the PPOP is provided in Section~\ref{sec:scalability}.

\section{Complexity}\label{sec:complexity}

In this section, we show that the extra requirements MOQ and MPA make the problem hard.  More precisely, each of the requirements itself puts the time complexity of the problem to the class of \mbox{$\mathcal{NP}$-complete} problems.

\newtheorem{proposition}{Proposition}
\begin{proposition}\label{prop1}
The PPOP with the MOQ requirement but without the MPA requirement is \mbox{$\mathcal{NP}$-complete}.
\end{proposition}
\begin{proof}

The problem is clearly in $\mathcal{NP}$ as the feasibility of a solution and its objective value can be verified in polynomial time.
We show \mbox{$\mathcal{NP}$-hardness} by a polynomial transformation from the independent set problem, known to be \mbox{$\mathcal{NP}$-complete}.

An instance of the independent set problem is given by a graph $G=(V,E)$ and a positive integer number $k \le |V|$, and the question is whether there is a subset of nodes $S \subseteq V$ of size $k$ such that no two vertices in $S$ are connected by an edge. Given an instance of the independent set problem, we construct a corresponding instance of the PPOP as follows. An example reduction for a graph with vertices $V=\{1,2,3\}$ and edges $E=\{\{1,2\},\{2,3\}\}$ is depicted in Figure~\ref{proof}.

\begin{figure*}[t]
\centering
\newcommand{\mat}[3] { 
 	\node[circle, draw, thick, minimum size=0.8cm] (m#1) at (#2,#3) {$#1$};
}
\newcommand{\matcs}[4] { 
 	\node[circle, draw, thick, minimum size=0.8cm] (mc#1) at (#2,#3) {$\widetilde{#1}$};
 	\node[font=\scriptsize] (d#1) at (#2+0.75,#3-1.1) {$h_{\widetilde{#1}}=#4$};
}
\newcommand{\kus}[3] { 
 	\node[draw, thick, minimum size=0.8cm] (k#1) at (#2,#3) {$#1$};
}

\subfloat[An example instance of the independent set problem.]{%
\begin{tikzpicture}[scale=0.65, draw=black, text=black]
\mat{1}{-5}{0}
\mat{2}{0}{-3}
\mat{3}{5}{0}
\draw[thick] (m1) -- (m2);
\draw[thick] (m2) -- (m3);
\end{tikzpicture}
\label{proof1}}
\hfil
\vspace{1em}
\subfloat[Corresponding instance of the PPOP. The circles are materials, the squares are recipes. The arrows depict the flows of materials and their quantities specified by recipes.]{%
\begin{tikzpicture}[scale=0.65, draw=black, text=black]
\mat{1}{-6.5}{0}
\matcs{1}{-3.5}{0}{1}
\kus{1}{-5}{-3}
\mat{2}{-1.5}{0}
\matcs{2}{1.5}{0}{1}
\kus{2}{0}{-3}
\mat{3}{3.5}{0}
\matcs{3}{6.5}{0}{1}
\kus{3}{5}{-3}

\mat{12}{-6.5}{-6}
\matcs{12}{-3.5}{-6}{0.5}
\kus{12}{-5}{-9}
\mat{13}{-1.5}{-6}
\matcs{13}{1.5}{-6}{1}
\kus{13}{0}{-9}
\mat{23}{3.5}{-6}
\matcs{23}{6.5}{-6}{0.5}
\kus{23}{5}{-9}

\mat{t}{0}{-12}

\draw[-stealth, thick] (m1) -- node[anchor=east] {$100$} (k1);
\draw[-stealth, thick] (mc1) -- node[anchor=west] {$1$} (k1);
\draw[-stealth, thick] (m2) -- node[anchor=east] {$100$} (k2);
\draw[-stealth, thick] (mc2) -- node[anchor=west] {$1$} (k2);
\draw[-stealth, thick] (m3) -- node[anchor=east] {$100$} (k3);
\draw[-stealth, thick] (mc3) -- node[anchor=west] {$1$} (k3);

\draw[-stealth, thick] (m12) -- node[anchor=east] {$100$} (k12);
\draw[-stealth, thick] (mc12) -- node[anchor=west] {$1$} (k12);
\draw[-stealth, thick] (m13) -- node[anchor=east] {$100$} (k13);
\draw[-stealth, thick] (mc13) -- node[anchor=west] {$1$} (k13);
\draw[-stealth, thick] (m23) -- node[anchor=east] {$100$} (k23);
\draw[-stealth, thick] (mc23) -- node[anchor=west] {$1$} (k23);

\draw[-stealth, thick] (k1) -- node[anchor=east] {$50$} (m12);
\draw[-stealth, thick] (k1) -- node[anchor=south, pos=0.3] {~$50$} (m13);
\draw[-stealth, thick] (k2) -- node[anchor=south, pos=0.2] {$50$} (m12);
\draw[-stealth, thick] (k2) -- node[anchor=south, pos=0.3] {~$50$} (m23);
\draw[-stealth, thick] (k3) -- node[anchor=south, pos=0.2] {$50$} (m13);
\draw[-stealth, thick] (k3) -- node[anchor=west] {$50$} (m23);

\draw[-stealth, thick] (k12) -- node[anchor=north east] {$100$} (mt);
\draw[-stealth, thick] (k13) -- node[anchor=west] {$100$} (mt);
\draw[-stealth, thick] (k23) -- node[anchor=north west] {$100$} (mt);

\mat{i}{10}{-8}
\kus{j}{10}{-10}
\mat{i'}{10}{-12}
\draw[-stealth, thick] (mi) -- node[anchor=west] {$q^+_{j,index^+(i)}$} (kj);
\draw[-stealth, thick] (kj) -- node[anchor=west] {$q^-_{j,index^-(i')}$} (mi');

\end{tikzpicture}
\label{proof2}}

\caption{Diagram for the reduction.}
\label{proof}
\end{figure*}
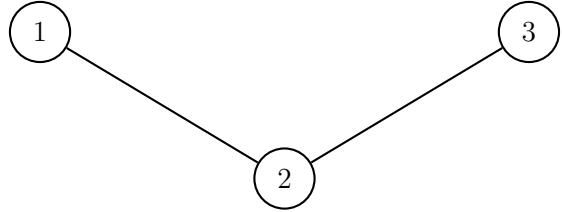
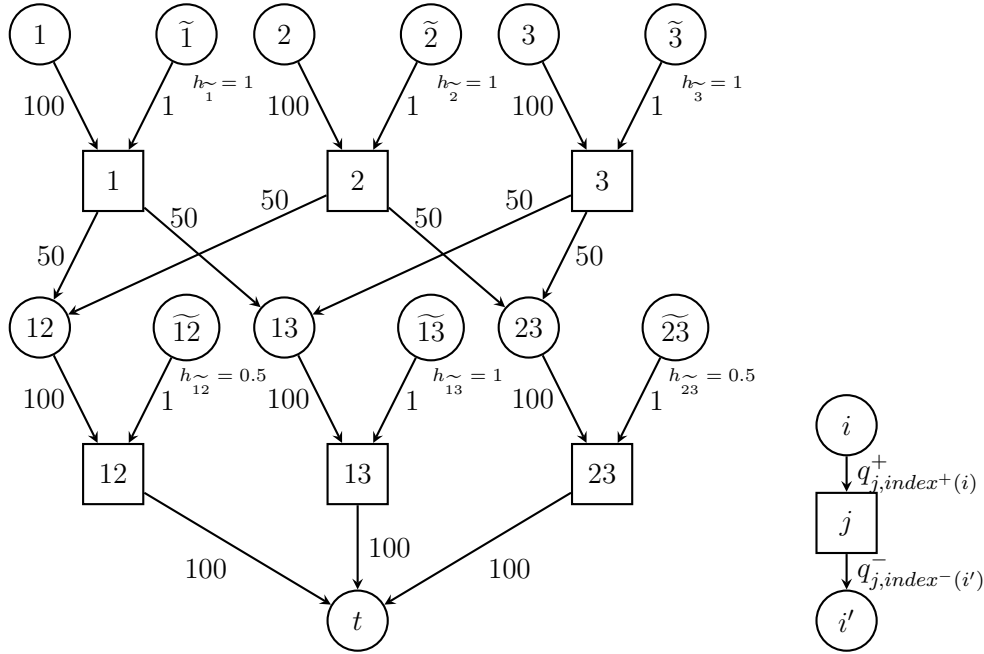

We set weights in the objective function as $w_0 = 1$, and $w_1 = w_2 = w_3 = w_4 = 0$, i.e., we penalize only the cost of materials to be bought. Next, we can consider $b^{min}_i$ to be an arbitrary constant equal for all materials. Without loss of generality, we will consider $b^{min}_i$ to be 100 for all the introduced materials. This choice also affects some other numbers in the reduction, as will be evident further.

For each $i \in V$, we create materials $i$ and $\widetilde{i}$, such that $c_{i} = 1$, $c_{\widetilde{i}} > k$, $h_{i} = 0$, $h_{\widetilde{i}} = 1$.
For each pair of vertices $i,j \in V, i<j$, we create materials $ij$ and $\widetilde{ij}$, such that $c_{\widetilde{ij}} > k$ and $h_{ij} = 0$. 
If $\{i,j\} \notin E$, we set $h_{\widetilde{ij}} = 1$, and if $\{i,j\} \in E$, we set $h_{\widetilde{ij}} = 0.5$.

Next, we create the set of recipes as follows. For each $i \in V$, we create a recipe $i$ with the input materials $M^+_{i} = \left( i, \widetilde{i}\right)$, their quantities $Q^+_{i} = \left( 100, 1\right)$, the output materials $M^-_{i} = \left( ji | j < i\right) \cup \left( ij | i < j\right)$, and their quantities $Q^-_{i} = \left( \frac{100}{n-1}, \frac{100}{n-1}, \dots, \frac{100}{n-1}\right)$, where $n=|V|$. 

Note that the values for $c_{ij}$ can be arbitrary as each material $ij$ is an output of some recipe and hence cannot be bought (recall that a material can be bought only if it cannot be produced by any recipe).

Next, for each pair of vertices $i,j \in V, i<j$, we create a recipe $ij$ with the input materials $M^+_{ij} = \left( ij, \widetilde{ij}\right)$, their quantities $Q^+_{ij} = \left( 2\frac{100}{n-1}, 1\right)$, the output material $M^-_{ij} = \left( t\right)$, and its quantity $Q^-_{ij} = \left( 2\frac{100}{n-1}\right)$. We set $d_{t} = 100k$.

The construction is done such that there is an independent set $S$ of size $k$ in the original graph if and only if there is a solution to the PPOP of cost $100k$.
If there is a solution of cost $100k$, then each recipe $i$ for which $z_i = 1$ determines node $i$ to be selected in $S$; other recipes have $z_i = 0$ and corresponding nodes are not selected in $S$.

'$\Leftarrow$': Suppose there is a solution to the PPOP of cost $100k$. We first show that there are exactly $k$ recipes from the set of recipes $\{1, 2, \dots, n\}$ for which $z_i = 1$, and for the other $n-k$ recipes it holds that $z_i = 0$, thereby fulfilling the requirement on the size of $S$ to be $k$.
If $z_i = 1$, then $b_i = 100$, and if $z_i = 0$, then $b_i = 0$. 
If $z_i > 1$ and hence $b_i > 100$, the available material $\widetilde{i}$ from stock would not suffice and we would need to buy it, in the minimum quantity of 100, which would increase the objective function by more than $100k$ because $c_{\widetilde{i}} > k$. 
To satisfy demand $d_{t} = 100k$, there must be at least $k$ recipes for which $z_i = 1$ because only these recipes for which $z_i = 1$ contribute to the production of material $t$ by 100 units.
If, for some $i$, it held that $z_i > 0$ and $z_i < 1$, then, in order to satisfy demand, there would have to be at least $k+1$ such recipes for which $z_i > 0$ and hence at least $k+1$ materials for which $b_i \ge 100$ due to MOQ, which would in turn raise the objective function to at least $100(k+1)$. Therefore, since the solution is of cost $100k$, there are exactly $k$ such recipes for which $z_i = 1$. 

The requirement that the nodes in $S$ must not be connected is ensured as follows. For each edge $\{i,j\} \in E$, we have $h_{\widetilde{ij}} = 0.5$, which ensures that $z_{ij} \le 0.5$ and hence either $z_i = 0$ or $z_j = 0$ (or both $z_i = z_j = 0$), therefore, at most one of the nodes $i$ and $j$ can be in the independent set $S$. Otherwise, having $h_{\widetilde{ij}} = 1$ if $\{i,j\} \notin E$ allows $z_{ij} = 1$ and hence both $z_i$ and $z_j$ can be equal to $1$, therefore, both nodes $i$ and $j$ are allowed to be in $S$. This works because otherwise, we would need to buy material $\widetilde{ij}$, which would again increase the objective function by more than $100k$ as $c_{\widetilde{ij}} > k$, and because materials $ij$ cannot be stored as it would fail to satisfy demand $d_{t} = 100k$.

'$\Rightarrow$': Suppose there is an independent set $S$ of size $k$. By similar reasoning, setting $z_i = 1$ for each node $i \in S$ and $z_i = 0$ for each node $i \notin S$ yields a solution to the PPOP of cost $100k$.

The reduction is polynomial as the created instance of the problem with MOQ is of size $\Theta(n^2)$. 

\end{proof}

Notice that the reduction in the proof fixes the parameter $b^{min}_i$ to a constant that is equal for all materials, and no group of alternatives is used. Therefore, we obtain the corollary that the problem remains \mbox{$\mathcal{NP}$-complete} even for constant $b^{min}_i$ and without the concept of alternatives.

\begin{proposition}\label{prop2}
The PPOP with the MPA requirement but without the MOQ requirement is \mbox{$\mathcal{NP}$-complete}.
\end{proposition}
\begin{proof}
To ease the explanation, we describe how to modify the reduction from Proposition~\ref{prop1} to use MPA instead of MOQ.
First, notice that for materials $\widetilde{i}$ and $\widetilde{ij}$, the MOQ requirements are not necessary. We can set prices $c_{\widetilde{i}}$ and $c_{\widetilde{ij}}$ to be large enough so that buying even a small quantity would exceed the permitted objective value. It remains to show how to replace MOQ for materials $i$ such that the flow of material $i$ to recipe $i$ is still either 0 or 100. 

For each triplet of materials $i$, $\widetilde{i}$, and recipe $i$, we replace it with another material flow structure illustrated in Figure~\ref{proofmpa}, which is constructed as follows. There are materials $i$, $\widetilde{i}$, and $\bar{i}$, with $c_i$ being large enough so that material $i$ cannot be bought, $h_i = 100$, $c_{\widetilde{i}} = 0$, $h_{\widetilde{i}} = 0$, $h_{\bar{i}} = 0$,  and $d_{\bar{i}} = 1$. Next, there are recipes $i$ and $\hat{i}$ defined as follows. For recipe $i$, $M^+_{i} = \left( i\right)$, $Q^+_{i} = \left(100\right)$, and the definitions of $M^-_{i}$ and $Q^-_{i}$ stay the same as in Proposition~\ref{prop1}. For recipe $\hat{i}$, $M^+_{\hat{i}} = \left(\right)$, $Q^+_{\hat{i}} = \left(\right)$, $M^-_{\hat{i}} = \left(\bar{i}\right)$, $Q^-_{\hat{i}} = \left(1\right)$, and the alternatives are defined as $A = (i, \widetilde{i})$, $q_{\hat{i}, A} = 2000$, and $(\hat{i}, A, q_{\hat{i}, A}) \in Alts$. Note that the number 2000 is chosen because 100 is 5 \% of 2000.

We need to ensure that material $i$ always flows in the quantity of 100 to either recipe $i$ or $\hat{i}$. To achieve this, we modify the objective weights such that $w_0 = w_1 = 1$, and $w_2 = w_3 = w_4 = 0$, and hence material $i$ cannot stay unused in stock. Note that material $\widetilde{i}$ is of cost $c_{\widetilde{i}} = 0$ so it can be bought in any quantity. To satisfy demand $d_{\bar{i}} = 1$, there are only two possible scenarios. One scenario is that material $i$ goes to recipe $i$ in the quantity of 100 and material $\widetilde{i}$ goes to recipe $\hat{i}$ in the quantity of 2000. The other scenario is that material $i$ in the quantity of 100 goes to recipe $\hat{i}$ and then material $\widetilde{i}$ goes to recipe $\hat{i}$ only in the quantity of 1900, therefore, material $i$ makes exactly 5~\% in the group of alternatives. Indeed, material $i$ cannot go to recipe $i$ in the quantity more than 0 and less than 100, since the remaining quantity from stock would need to go to recipe $\hat{i}$, which would then make more than 0 but less than 5~\% in the group of alternatives, and, finally, material $i$ cannot go to recipe $i$ in the quantity more than 100 because it is too expensive to buy. 
The rest of the proof from Proposition~\ref{prop1} applies without any modifications.

\begin{figure*}[t]
\centering
\newcommand{\mat}[3] { 
 	\node[circle, draw, thick, minimum size=0.8cm] (m#1) at (#2,#3) {$#1$};
}
\newcommand{\matcs}[4] { 
 	\node[circle, draw, thick, minimum size=0.8cm] (mc#1) at (#2,#3) {$\widetilde{#1}$};
 	\node[font=\scriptsize] (d#1) at (#2+0.75,#3-1) {$h_{\widetilde{#1}}=#4$};
}
\newcommand{\kus}[3] { 
 	\node[draw, thick, minimum size=0.8cm] (k#1) at (#2,#3) {$#1$};
}

\begin{tikzpicture}[scale=0.65, draw=black, text=black]
\mat{i}{-7.5}{0}
\matcs{i}{-4.5}{0}{1}
\kus{i}{-6}{-3}

\draw[-stealth, thick] (mi) -- node[anchor=east] {$100$} (ki);
\draw[-stealth, thick] (mci) -- node[anchor=west] {$1$} (ki);

\draw[dotted, thick] (-3,0.5) -- (-3,-6.8);

\node[circle, draw, thick, minimum size=0.8cm] (a) at (-0.5,0) {$i$};
\node[font=\scriptsize] (aa) at (-0.5-1.25,-0.9) {$h_{i}=100$};
\node[draw, thick, minimum size=0.8cm] (b) at (-0.5,-3) {$i$};
\node[circle, draw, thick, minimum size=0.8cm] (c) at (5.5,0) {$\widetilde{i}$};
\node[font=\scriptsize] (cc) at (5.5+0.75,-1.0) {$c_{\widetilde{i}}=0$};
\node[draw, thick, minimum size=0.8cm] (d) at (2.5,-3) {$\hat{i}$};
\node[circle, draw, thick, minimum size=0.8cm] (e) at (2.5,-6) {$\bar{i}$};
\node[font=\scriptsize] (ee) at (2.5+0.75,-6-0.9) {$d_{\bar{i}}=1$};

\draw[-stealth, thick] (a) -- node[anchor=east] {$100$} (b);
\draw[-stealth, dashed] (a) -- (d);
\draw[-stealth, dashed] (c) -- (d);
\draw[-stealth, thick] (d) -- node[anchor=east] {$1$} (e);

\node[] (f) at (2.5,-1.5) {$2000$};
\draw[dashed] (1.5,-2.5) to [controls=+(90:1) and +(90:1)] (3.5,-2.5);

\end{tikzpicture}
\caption{Modification of the reduction for the problem with MPA. The part on the left side of the figure is to be replaced by the part on the right side of the figure. The group of alternatives is depicted by dashed lines.}
\label{proofmpa}
\end{figure*}
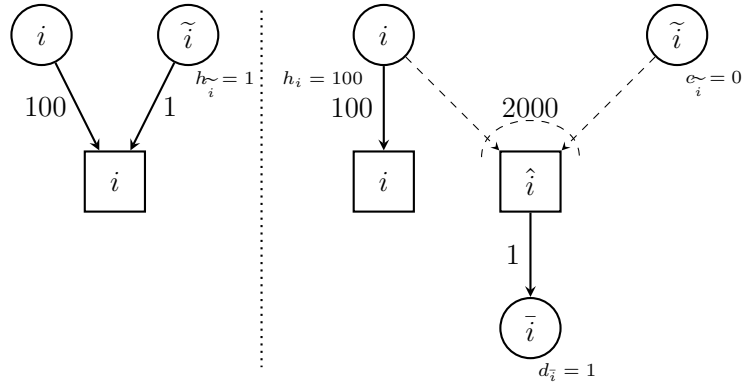
\end{proof}

\section{Experiments}\label{sec:experiments}

We implemented our proposed solution in Python with Gurobi Optimizer version 10.0.0 build v10.0.0rc2 (win64) as the solver for the mathematical model, using the PWL concept (see Equation (\ref{eq:PWLObjective})). \textcolor{black}{Some experiments were carried out using an open-source mixed-integer program solver CBC version 2.10.10, as this is the solver used in the algorithm currently deployed in the test operation at the meat processing company.} We conducted benchmarks on a Dell PC with an Intel\textregistered~Core\texttrademark~i7-4610M processor running at 3.00 GHz with 16 GB of RAM.

\textcolor{black}{We performed the experimental analysis on real data from the meat processing company. The basic set of recipes consists of 507 materials, 246 recipes, and non-zero demands for 272 materials. The biggest demand is 18125, the smallest non-zero demand is 0.745, and the average demand is 127.229. The stock is also taken from the data of the company, containing 42 different materials in a total quantity of 211069. The biggest number of distinct expiration dates of one material is 4.
For the scalability experiments (Section~\ref{sec:scalability}), we assume an extended set of recipes consisting of 1130 materials and 1131 recipes. This dataset also contains some additional recipes that the company did not use when we were conducting experiments. Nevertheless, it makes it even more challenging to solve and indicates how complex scenarios the company may solve.
}

Below, we analyze the objective weights and their impact on the solution. The subsequent subsection shows an analysis allowing the company to determine the number of hogs to be slaughtered. The next subsection analyses the impact of the MOQ requirement on the solution runtime. The last subsection studies the scalability of the algorithm using a larger dataset with randomly generated demands.




\subsection{Objective Weights}
\label{sec:weights}

We evaluated several combinations of objective weights and their impact on the objective values. Since we deem functions $f_0$ and $f_1$ more significant for the benefit of the company, we intentionally choose larger values for these two functions. \textcolor{black}{Functions $f_2$, $f_3$, and $f_4$ are auxiliary, and their uses depend on the actual need of the purchasing manager.} For now, $b^{min}_i$ is set to 0, $\forall i \in M$.

The results are shown in Table~\ref{tbl:weights}. The column $w$ is the vector of weights in the compound objective function. For each $l \in \{0,1,2,3,4\}$, $f_l$ are the values of the function $f_l$ evaluated alone (without the weights) and $t_l$ are the relative deteriorations in the value of $f_l$ against the best possible value (in percentage). More precisely, if $f^{*}_l$ is the best possible value of $f_l$ that would be attained when optimizing $f_l$ alone, the value of $t_l$ is calculated as $t_l = 100\cdot(\frac{f_l}{f^{*}_l} - 1)$. Note that the value of $f^{*}_l$ is the value $f_l$ on the row in the table where $w_l = 1$ and the other weights are 0.

\begin{table}[t]
\tiny
  \begin{center}
    \caption{Experimental evaluation of the impact of the objective weights}
    \label{tbl:weights}
    \begin{filecontents*}{BM.csv}
,w0,w1,w2,w3,w4,cw0,cw1,cw2,cw3,cw4,f0,f1,f2,f3,f4,f,t0,t1,t2,t3,t4,t,to,iterations,consP
0,1,0,0,0,0,1.0,0.0,0.0,0.0,0.0,9588513.816582575,25193577.872242942,2665068.940385695,2484370.8557359288,73756.92348389984,9588513.81658258,-3.3306690738754696e-14,130.452268808695,1238.342521759141,3581.6694518809372,84.96096344790433,5035.425205896678,55.35425205896677,1,0
1,0,1,0,0,0,0.0,1.0,0.0,0.0,0.0,9928696.618717808,10932232.51933217,283607.1718176007,150108.62429784393,105302.19664171751,10932232.51933217,3.5478157370635532,0.0,42.42165812956957,122.4508210057619,164.06735563318708,332.4876505055821,8.32487650505582,1,0
2,0,0,1,0,0,0.0,0.0,1.0,0.0,0.0,10963092.210965794,12411588.300088475,199132.05304742782,150524.02875885798,39877.02167472436,199132.0530474278,14.33567725590581,13.532055580964485,2.220446049250313e-14,123.06642230005313,-3.3306690738754696e-14,150.9341551369234,6.509341551369234,1,0
3,0,0,0,1,0,0.0,0.0,0.0,1.0,0.0,10963468.71557745,12140476.772116585,204771.08090692564,67479.46517758898,110872.10515386377,67479.465177589,14.339603877057506,11.052127281850233,2.831803204557315,-2.220446049250313e-14,178.035070066777,206.25860443024203,7.062586044302421,1,0
4,0,0,0,0,1,0.0,0.0,0.0,0.0,1.0,13932890.355524972,15311653.978478748,258946.71739778045,209024.3139923797,39877.02167472436,39877.021674724376,45.30813243893059,40.05971745846209,30.037687773000776,209.75988538480595,-3.3306690738754696e-14,325.16542305519937,8.251654230551994,1,0
5,1,1,1,1,1,0.2,0.2,0.2,0.2,0.2,9917962.100779695,10935704.897962745,283465.03389483585,151109.51337828112,104136.03881229616,4278475.516965573,3.4358638940203745,0.03176275865368616,42.35027940344804,123.93407087711626,161.14297015893158,330.8949470921699,8.3089494709217,1,0
6,10,1,1,1,1,0.7142857142857143,0.07142857142857142,0.07142857142857142,0.07142857142857142,0.07142857142857142,9887633.15113599,10994118.65354822,285259.2091107626,152843.70685551863,104136.03881229616,7886620.651406195,3.1195588834226573,0.5660887115839452,43.25127710244703,126.50402823032519,161.14297015893158,334.5839230867104,8.345839230867103,1,0
7,1,10,1,1,1,0.07142857142857142,0.7142857142857143,0.07142857142857142,0.07142857142857142,0.07142857142857142,9928696.618717808,10932232.51933217,283607.1718176007,150108.62429784393,105302.19664171751,8556431.414628334,3.5478157370635532,0.0,42.42165812956957,122.4508210057619,164.06735563318708,332.4876505055821,8.32487650505582,1,0
8,10,10,1,1,1,0.43478260869565216,0.43478260869565216,0.043478260869565216,0.043478260869565216,0.043478260869565216,9917962.100779708,10935640.534326369,283572.09229105554,150074.71742708914,105302.19664171751,9090216.319887856,3.4358638940205077,0.031174007579615193,42.40404191660518,122.40057331831275,164.06735563318708,332.3390087697052,8.323390087697051,1,0
9,100,1,1,1,1,0.9615384615384616,0.009615384615384616,0.009615384615384616,0.009615384615384616,0.009615384615384616,9588513.816582575,24804245.854673807,2666014.306670497,2454006.44038046,104136.03881229616,9508459.464411493,-3.3306690738754696e-14,126.89094666447005,1238.8172651619905,3536.671443560099,161.14297015893158,5063.522625545491,55.63522625545491,1,0
10,1,100,1,1,1,0.009615384615384616,0.9615384615384616,0.009615384615384616,0.009615384615384616,0.009615384615384616,9928696.618717808,10932232.51933217,283607.1718176007,150108.62429784393,105302.19664171751,10612413.13985281,3.5478157370635532,0.0,42.42165812956957,122.4508210057619,164.06735563318708,332.4876505055821,8.32487650505582,1,0
11,100,10,1,1,1,0.8849557522123894,0.08849557522123894,0.008849557522123894,0.008849557522123894,0.008849557522123894,9887633.15113599,10994054.289911855,285366.26750698325,151808.91090432764,105302.19664171751,9727843.675997974,3.1195588834226573,0.5654999605100075,43.305039615604635,124.97053067152315,164.06735563318708,336.02798476424755,8.360279847642476,1,0
12,10,100,1,1,1,0.08849557522123894,0.8849557522123894,0.008849557522123894,0.008849557522123894,0.008849557522123894,9928696.618717808,10932232.51933217,283607.1718176007,150108.62429784393,105302.19664171751,10557957.841709312,3.5478157370635532,0.0,42.42165812956957,122.4508210057619,164.06735563318708,332.4876505055821,8.32487650505582,1,0
13,100,100,1,1,1,0.49261083743842365,0.49261083743842365,0.0049261083743842365,0.0049261083743842365,0.0049261083743842365,9917962.100779708,10935640.534326369,283572.09229105565,150074.71742708923,105302.19664171751,10275365.578901317,3.4358638940205077,0.031174007579615193,42.40404191660525,122.4005733183129,164.06735563318708,332.33900876970534,8.323390087697053,1,0
\end{filecontents*}
\pgfplotstabletypeset[
create on use/w/.style={
    create col/assign/.code={%
        \edef\entry{$(\thisrow{w0},\thisrow{w1},\thisrow{w2},\thisrow{w3},\thisrow{w4})$}
        \pgfkeyslet{/pgfplots/table/create col/next content}{\entry}
    }
},
      columns={w,f0,f1,f2,f3,f4,t0,t1,t2,t3,t4},
      multicolumn names,
      every head row/.style={before row={\toprule}, after row=\midrule},
      col sep=comma,
      display columns/0/.style={string type, column name=$w$, column type=c},
      display columns/1/.style={column name=$f_0$, column type=r, sci, sci zerofill},
      display columns/2/.style={column name=$f_1$, column type=r, sci, sci zerofill},
      display columns/3/.style={column name=$f_2$, column type=r, sci, sci zerofill},
      display columns/4/.style={column name=$f_3$, column type=r, sci, sci zerofill},
      display columns/5/.style={column name=$f_4$, column type=r, sci, sci zerofill},
      display columns/6/.style={column name=$t_0$, column type=r, fixed, fixed zerofill},
      display columns/7/.style={column name=$t_1$, column type=r, fixed, fixed zerofill},
      display columns/8/.style={column name=$t_2$, column type=r, fixed, fixed zerofill},
      display columns/9/.style={column name=$t_3$, column type=r, fixed, fixed zerofill},
      display columns/10/.style={column name=$t_4$, column type=r, fixed, fixed zerofill},
      every last row/.style={after row=\bottomrule},
    ]{BM.csv}
  \end{center}
\end{table}

Relative deteriorations $t_l$ in Table~\ref{tbl:weights} show that weights $(1, 1, 1, 1, 1)$, $(10, 10, 1, 1, 1)$, and $(100, 100, 1, 1, 1)$ provide a good compromise concerning the best possible $f_l$. Other scenarios favor some $f_l$ at the expense of another. 
Regarding the fact that functions $f_0$ and $f_1$ are denominated in fiat currency, whereas the other functions are numbers reflecting the quality of the solution or planners preferences (such as freshness of the products) that will not be immediately projected to financial benefits, we consider for further experiments the combination of weights $w = (100, 100, 1, 1, 1)$. 
\textcolor{black}{Nevertheless, in the real implementation, the company has several predefined sets of weights designed based on the results presented in Table~\ref{tbl:weights}. For example, the manager can selectively turn on and off auxiliary objective functions $f_2$, $f_3$, and $f_4$. However, $f_0$ and $f_1$ are always the main objectives with the highest weight as both define immediate costs. Which set of weights is used reflects the actual buying strategy of the company and the actual state of the stock.}

\subsection{Required Number of Hogs to be Cut}
\label{sec:numberOfHogs}

\textcolor{black}{The purchasing manager must also determine the amount of material coming from slaughter and purchase. Sometimes, the company even} has to decide to buy and slaughter more hogs than the optimal number, i.e., when it would be financially beneficial to buy (and resell) more components and slaughter less or even not at all. The motivation is to maintain the brand and reputation of the company as a full-fledged meat processing company that involves slaughtering, as well as to keep the number of employees stable. \textcolor{black}{Nevertheless, the company needs to understand the impact of this decision on its profit.} Therefore, we evaluate the impact of the required number of hogs to be slaughtered on the solution quality \textcolor{black}{as it is one of the key decisions for the company.}

This analysis is realized by adding a constraint ensuring that the corresponding variable $z_j$ is of the requested value. 
For now, $b^{min}_i$ is still set to 0, $\forall i \in M$, and the vector of weights used is $w = (100, 100, 1, 1, 1)$.
The results are depicted in Table~\ref{tbl:cuts}. The column 'hogs cut' is the number of hogs required to be bought, slaughtered, and processed.
The dependence of the individual objective functions on the number of cut hogs is depicted in Figure~\ref{fig:BMmcut}.

Function $f_4$ is constant, meaning that the number of hogs cut does not directly affect what is processed from stock. Otherwise, we can observe that too low numbers of cut hogs are detrimental since in order to satisfy demand, we need to buy other expensive materials. As the number of hogs increases, it also increases the flexibility of production, meaning that more effective and cheaper ways of satisfying demand can be found. From a certain point, all functions (except for $f_4$) are linearly increasing as every extra hog only raises overproduction. Interestingly, the threshold from which the functions start increasing is larger for functions $f_2$ and $f_3$ than for functions $f_0$ and $f_1$. If the number of cut hogs is not enforced by the constraint, the optimum number of cut hogs, for given objective weights, is 689.75.

All instances were solved in one iteration and negligible runtime, i.e., less than one second, using the commercial ILP solver. Note that for less than 400 cut hogs, the model becomes infeasible as the demand cannot be satisfied.

\begin{table*}[t]
\footnotesize
  \begin{center}
    \caption{Experimental evaluation of the impact of the required quantity of hogs to be cut}
    \label{tbl:cuts}
    \begin{filecontents*}{BMmcut3.csv}
,minCut,w0,w1,w2,w3,w4,cw0,cw1,cw2,cw3,cw4,f0,f1,f2,f3,f4,f,t0,t1,t2,t3,t4,t,to,iterations,consP,pigsBought,pigsCut
0,400,100,100,1,1,1,0.49261083743842365,0.49261083743842365,0.0049261083743842365,0.0049261083743842365,0.0049261083743842365,10308274.18796059,11558705.781552719,344275.99137248896,208749.25885348787,105302.19664171751,10775154.307380293,7.506485208721725,5.730515346364218,72.88828498669267,209.35227228626113,164.06735563318708,459.54491346122677,9.595449134612268,1,0,400.0,400.0
1,500,100,100,1,1,1,0.49261083743842365,0.49261083743842365,0.0049261083743842365,0.0049261083743842365,0.0049261083743842365,9960900.369691769,11169811.324177492,320155.61935554986,185435.23518342525,105302.19664171751,10412226.908562107,3.883673322399317,2.1731956800698837,60.7755328466872,174.8024672326704,164.06735563318708,405.7022247150139,9.057022247150138,1,0,500.0,500.0
2,600,100,100,1,1,1,0.49261083743842365,0.49261083743842365,0.0049261083743842365,0.0049261083743842365,0.0049261083743842365,9915247.847669993,11056600.772277562,300569.7051668769,166504.08923914513,105302.19664171751,10333779.497467013,3.407556555035174,1.1376290499261144,50.93989167845794,146.74779031064963,164.06735563318708,366.30022322725597,8.663002232272559,1,0,600.0,600.0
3,700,100,100,1,1,1,0.49261083743842365,0.49261083743842365,0.0049261083743842365,0.0049261083743842365,0.0049261083743842365,9944330.974824524,10945265.678904912,281956.9228484259,148513.54483087623,105302.19664171751,10293080.97555303,3.710868702369563,0.11921773114222933,41.59293721602495,120.08702119974708,164.06735563318708,329.5774004824709,8.29577400482471,1,0,700.0,700.0
4,800,100,100,1,1,1,0.49261083743842365,0.49261083743842365,0.0049261083743842365,0.0049261083743842365,0.0049261083743842365,10213846.916250754,11063695.527857592,265133.3880688607,132252.4417837466,105302.19664171751,10484024.297720835,6.521689509240836,1.2025266412230673,33.144505875060304,95.98916712764603,164.06735563318708,300.9252447863573,8.009252447863574,1,0,800.0,800.0
5,900,100,100,1,1,1,0.49261083743842365,0.49261083743842365,0.0049261083743842365,0.0049261083743842365,0.0049261083743842365,10629879.786791269,11238757.928053191,233208.9977552945,101395.34646551947,105302.19664171751,10774894.965641923,10.860556600624903,2.803868360638795,17.11273709398784,50.26104044937583,164.06735563318708,245.10555813781446,7.451055581378144,1,0,900.0,900.0
6,1000,100,100,1,1,1,0.49261083743842365,0.49261083743842365,0.0049261083743842365,0.0049261083743842365,0.0049261083743842365,11063950.776000954,11492882.217612047,207784.11867819563,76820.47321152662,105302.19664171751,11113661.114038587,15.387545845391838,5.128409931717459,4.344888478956799,13.842741653856393,164.06735563318708,202.77094154310956,7.0277094154310955,1,0,1000.0,1000.0
7,1100,100,100,1,1,1,0.49261083743842365,0.49261083743842365,0.0049261083743842365,0.0049261083743842365,0.0049261083743842365,11498021.765210643,11904472.47199943,208544.61737034254,77555.55973125195,105302.19664171751,11530250.374850994,19.914535090158815,8.89333400975496,4.726795198898959,14.932090121261643,164.06735563318708,212.53411005326146,7.125341100532614,1,0,1099.9999999999998,1100.0
8,1200,100,100,1,1,1,0.49261083743842365,0.49261083743842365,0.0049261083743842365,0.0049261083743842365,0.0049261083743842365,11932092.754420329,12320545.638557386,210059.24725680103,79019.56566380143,105302.19664171751,11949055.27244993,24.441524334925766,12.699264461949312,5.487411013017929,17.1016478210694,164.06735563318708,223.7972032641495,7.237972032641495,1,0,1200.0,1200.0
9,1300,100,100,1,1,1,0.49261083743842365,0.49261083743842365,0.0049261083743842365,0.0049261083743842365,0.0049261083743842365,12366163.743630022,12794179.193848131,221256.88743843784,89842.86376192064,105302.19664171751,12396308.845791422,28.968513579692768,17.031714896507744,11.1106344018562,33.14104302023853,164.06735563318708,254.3192615314823,7.543192615314823,1,0,1299.9999999999998,1300.0
10,1400,100,100,1,1,1,0.49261083743842365,0.49261083743842365,0.0049261083743842365,0.0049261083743842365,0.0049261083743842365,12800234.732839707,13267812.749138875,232454.52762007466,100666.16186003978,105302.19664171751,12843562.419132914,33.49550282445972,21.364165331066154,16.733857790694472,49.18043821940754,164.06735563318708,284.841319798815,7.84841319798815,1,0,1400.0,1400.0
11,1500,100,100,1,1,1,0.49261083743842365,0.49261083743842365,0.0049261083743842365,0.0049261083743842365,0.0049261083743842365,13235406.06074014,13742230.757871496,243745.58725329916,111579.75653804246,105302.19664171751,13291745.36651034,38.03396766087512,25.7037913671359,22.403994496679868,65.35364683817897,164.06735563318708,315.5627559960569,8.15562755996057,1,0,1500.0000000000002,1500.0
12,1600,100,100,1,1,1,0.49261083743842365,0.49261083743842365,0.0049261083743842365,0.0049261083743842365,0.0049261083743842365,13671506.060740143,14217315.26747103,255112.407870128,122566.57962396037,105302.19664171751,13740714.847316526,42.58211775318455,30.049514061557293,28.112176802278643,81.63537500096649,164.06735563318708,346.44653925117404,8.464465392511741,1,0,1600.0,1600.0
13,1700,100,100,1,1,1,0.49261083743842365,0.49261083743842365,0.0049261083743842365,0.0049261083743842365,0.0049261083743842365,14107606.060740143,14692399.777070569,266479.2284869568,133553.40270987822,105302.19664171751,14189684.32812271,47.13026784549395,34.39523675597873,33.820359107877394,97.91710316375395,164.06735563318708,377.33032250629105,8.77330322506291,1,0,1700.0,1700.0
14,1800,100,100,1,1,1,0.49261083743842365,0.49261083743842365,0.0049261083743842365,0.0049261083743842365,0.0049261083743842365,14543706.060740143,15167484.286670102,277846.0491037855,144540.22579579608,105302.19664171751,14638653.808928898,51.678417937803346,38.740959450400126,39.528541413476106,114.19883132654141,164.06735563318708,408.21410576140806,9.08214105761408,1,0,1800.0,1800.0
15,1900,100,100,1,1,1,0.49261083743842365,0.49261083743842365,0.0049261083743842365,0.0049261083743842365,0.0049261083743842365,14979806.060740143,15642568.796269635,289212.8697206144,155527.048881714,105302.19664171751,15087623.289735092,56.226568030112745,43.086682144821495,45.236723719074924,130.48055948932893,164.06735563318708,439.0978890165252,9.390978890165252,1,0,1900.0,1900.0
16,2000,100,100,1,1,1,0.49261083743842365,0.49261083743842365,0.0049261083743842365,0.0049261083743842365,0.0049261083743842365,15415906.060740143,16117653.30586917,300579.69033744314,166513.8719676319,105302.19664171751,15536592.770541277,60.77471812242214,47.43240483924291,50.94490602467363,146.7622876521164,164.06735563318708,469.9816722716422,9.699816722716422,1,0,2000.0,2000.0
\end{filecontents*}
\pgfplotstabletypeset[
      columns={minCut,f0,f1,f2,f3,f4,t0,t1,t2,t3,t4},
      multicolumn names,
      every head row/.style={before row={\toprule}, after row=\midrule},
      col sep=comma,
      display columns/0/.style={column name=hogs cut, column type=r},
      display columns/1/.style={column name=$f_0$, column type=r, sci, sci zerofill},
      display columns/2/.style={column name=$f_1$, column type=r, sci, sci zerofill},
      display columns/3/.style={column name=$f_2$, column type=r, sci, sci zerofill},
      display columns/4/.style={column name=$f_3$, column type=r, sci, sci zerofill},
      display columns/5/.style={column name=$f_4$, column type=r, sci, sci zerofill},
      display columns/6/.style={column name=$t_0$, column type=r, fixed, fixed zerofill},
      display columns/7/.style={column name=$t_1$, column type=r, fixed, fixed zerofill},
      display columns/8/.style={column name=$t_2$, column type=r, fixed, fixed zerofill},
      display columns/9/.style={column name=$t_3$, column type=r, fixed, fixed zerofill},
      display columns/10/.style={column name=$t_4$, column type=r, fixed, fixed zerofill},
      every last row/.style={after row=\bottomrule},
    ]{BMmcut3.csv}
  \end{center}
\end{table*}

\begin{figure*}[t]
\centering
\begin{tikzpicture}[draw=black, text=black]
\begin{axis}[
    scale only axis,
    ylabel={cost [currency]},
    xlabel={hogs [-]},
    axis y line*=left,
    width=0.8\textwidth,
    height=\axisdefaultheight,
    enlargelimits=false,
    ]
\addplot[solid,color=red,mark=*] table [x=minCut, y=f0, col sep=comma] {BMmcut3.csv}; \label{f0}
\addplot[densely dotted,color=orange,mark=square*] table [x=minCut, y=f1, col sep=comma] {BMmcut3.csv}; \label{f1}
\end{axis}
\begin{axis}[
    scale only axis,
    ylabel={cost [-]},
    axis y line*=right,
    axis x line=none,
    width=0.8\textwidth,
    height=\axisdefaultheight,
    enlargelimits=false,
    legend style={ at={(1.05,1)}, anchor=north west,align=left},
    ]
    \addlegendimage{/pgfplots/refstyle=f0}\addlegendentry{$f_0$}
    \addlegendimage{/pgfplots/refstyle=f1}\addlegendentry{$f_1$}
\addplot[densely dashed,color=blue,mark=diamond*] table [x=minCut, y=f2, col sep=comma] {BMmcut3.csv}; \addlegendentry{$f_2$}
\addplot[loosely dashed,color=violet,mark=triangle*] table [x=minCut, y=f3, col sep=comma] {BMmcut3.csv}; \addlegendentry{$f_3$}
\addplot[dashdotdotted,color=green,mark=star] table [x=minCut, y=f4, col sep=comma] {BMmcut3.csv}; \addlegendentry{$f_4$}
\end{axis}
\end{tikzpicture}
\caption{Dependence of the costs on the required number of cut hogs. The y-axis on the left is for $f_0$ and $f_1$, on the right for $f_2$, $f_3$, and $f_4$.}
\label{fig:BMmcut}
\end{figure*}
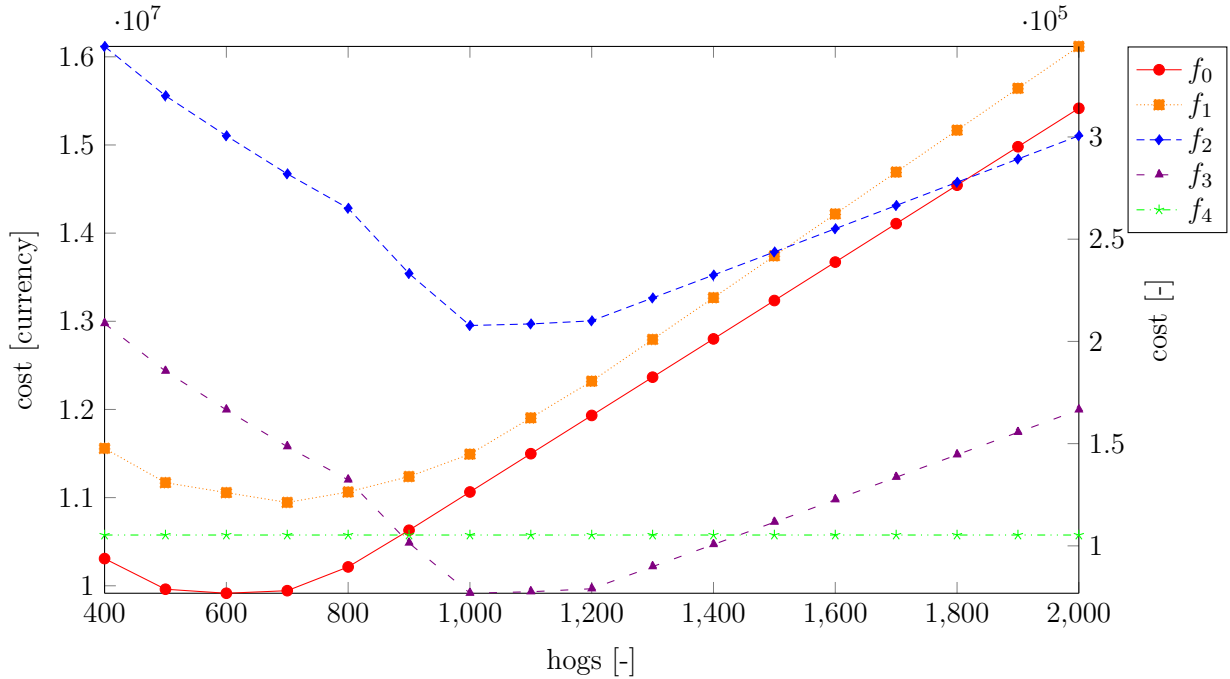

\subsection{Scalability}
\label{sec:scalability}
\textcolor{black}{To assess the scalability and compare the iterative model to the global model, we perform the experimental benchmarking on the extended set of recipes as it indicates how complex scenarios the company may solve in the future.}
\textcolor{black}{Since randomly creating more materials and recipes would likely carry the data further away from real-life needs, we investigate the scalability in terms of (i) the minimum order quantities $b^{min}_i$, and (ii) the number of different demands, i.e., the number of nonzero $d_i$.}

\textcolor{black}{Unlike the previous experiments, we focus on the use of the open-source mixed-integer program solver CBC to show that it is not necessary to invest in an expensive commercial solver. As will be shown below, the iterative model exploiting an open-source solver can provide a sufficiently fast solution to the user. Based on the requirement of the end users from the meat processing company, the time limit to solve a single instance was set to 60 seconds.
}

\subsubsection{Minimum Order Quantity}

We first assess the effect of MOQ on the CPU time, solution quality, and the number of iterations and added constraints the iterative model needs. For now, the required number of hogs to be cut is retracted, and the vector of weights used is again $w = (100, 100, 1, 1, 1)$.

\textcolor{black}{
The results are depicted in Table~\ref{tbl:moq}. The column 'moq' shows the value of $b^{min}_i$ that is set equal for all materials $i \in M$. Columns $CPUtime_{global}^{Gurobi}$ and $CPUtime_{global}^{CBC}$, respectively, denote the time Gurobi and CBC solvers need to solve the instance using the global model. Column $optimal_{global}^{CBC}$ indicates by 1 if the solution found by the CBC solver was optimal, and 0 indicates the timeout of the solver. The results obtained by the iterative model are summarized by columns $CPUtime_{iter}^{CBC}$ denoting the runtime of the algorithm, $optimal_{iter}^{CBC}$ indicates whether the algorithm found an optimal solution, $\#iter$ is the number of algorithm iterations, and $\#consB$ and $\#consP$ are numbers of MOQ and MPA constraints added by the algorithm, respectively.
}

\textcolor{black}{
The table shows that Gurobi can solve the global model very quickly, while CBC timeouts six times out of eight. Nevertheless, the iterative model with CBC always found the optimal solution before the timeout. The worst-case runtime was less than twenty seconds, which is much less than what is demanded by users of the algorithm. One can notice from the table that the number of added constraints and the number of iterations are rising with increasing $b^{min}_i$.
}

\begin{table*}[b]
\tiny
  \begin{center}
    \caption{Experimental evaluation of the impact of MOQ}
    \label{tbl:moq}
    \begin{filecontents*}{BMmoqCBCKompletkus.csv}
,moq,runtimeGurobi,runtimeCBC,statusCBC,runtimeCBCiter,statusCBCiter,iterations,consB,consP
1,0,0.2,60.5,0,0.6,1,1,0,0
2,10,0.2,60.7,0,10.8,1,9,14,54
3,20,0.6,50.7,1,10.7,1,9,22,54
4,50,0.2,25.5,1,9.0,1,9,29,53
5,100,0.1,60.7,0,11.0,1,10,37,54
6,200,0.4,61.5,0,14.6,1,12,47,41
7,500,0.2,60.4,0,9.7,1,10,65,53
8,1000,0.2,60.7,0,19.9,1,16,74,61
\end{filecontents*}
\pgfplotstabletypeset[
      columns={moq,runtimeGurobi,runtimeCBC,statusCBC,runtimeCBCiter,statusCBCiter,iterations,consB,consP},
      multicolumn names,
      every head row/.style={before row={\toprule}, after row=\midrule},
      col sep=comma,
      display columns/0/.style={column name=moq, column type=r, column type/.add={}{|}},
      display columns/1/.style={column name=$CPUtime_{global}^{Gurobi}$, column type=r},
      display columns/2/.style={column name=$CPUtime_{global}^{CBC}$, column type=r},
      display columns/3/.style={column name=$optimal_{global}^{CBC}$, column type=r, column type/.add={}{|}},
      display columns/4/.style={column name=$CPUtime_{iter}^{CBC}$, column type=r, precision=1, fixed, fixed zerofill},
      display columns/5/.style={column name=$optimal_{iter}^{CBC}$, column type=r},
      display columns/6/.style={column name=$\#iter$, column type=r},
      display columns/7/.style={column name=$\#consB$, column type=r},
      display columns/8/.style={column name=$\#consP$, column type=r},
      every last row/.style={after row=\bottomrule},
    ]{BMmoqCBCKompletkus.csv}
  \end{center}
\end{table*}

\subsubsection{Demand}
\label{sec:scalability_demand}


Next, we investigate the scalability in terms of the number of different demands, i.e., the number of nonzero $d_i$.
To make sure that the random demands are of a reasonable magnitude, we first performed an auxiliary run of the problem instance where we enforced all recipes and all materials in alternatives to be used in non-zero quantity. That is, we set $z_j \ge 0.01, \forall j \in J$, and $\hat{z}_{j,A,i} \ge 0.05 \cdot z_j, \forall (j,A,q_{j,A}) \in Alts, \forall i \in M$. From the solution, we stored the value of $s^{new}_i$, $\forall i \in M$, as the value $\widetilde{d_i}$, which will be used for generating random demand for material $i$.

In this set of benchmarks, we initially set all demands to zero and then we iteratively add demand for each material. Each such added demand is a randomly generated float number following a triangular distribution $\sim Triang(a, b, c)$, which is a continuous probability distribution with lower limit $a$, upper limit $b$, and mode $c$ (i.e., it gives a random floating point number within a range with a bias towards one extreme). The parameter $a$ is $0.1 \cdot \widetilde{d_i}$, $b$ is $5 \cdot \widetilde{d_i}$, and $c$ is $\widetilde{d_i}$.
Finally, we set $b^{min}_i$ to 100, $\forall i \in M$, and the vector of weights used is again $w = (100, 100, 1, 1, 1)$. 

\textcolor{black}{The result for the iterative model is shown in Table~\ref{tbl:BMscal}, which compares results obtained by the global model using Gurobi and CBC, and the iterative model exploiting CBC. Each row in the table refers to 40 instances having nonzero demand from the specified range (one instance for each value in the range). The values in the table are average values over 40 instances. The columns have the same meaning as in Table~\ref {tbl:moq} except column $optimal$ that in Table~\ref{tbl:BMscal} indicates the percentage of instances solved to optimality. It comes as no surprise that the commercial solver can solve all instances formulated using the global model very quickly, while the open-source solver often struggles to find the optimal solution before the timeout sixty seconds (see column $optimal_{global}^{CBC}$). On the other hand, the iterative model exploiting CBC can solve all instances to optimality within a few seconds (see columns $CPUtime_{iter}^{CBC}$ and $optimal_{iter}^{CBC}$). The average number of iterations the algorithm needs varies between 3 and 10.2 while the average number of added constraints of both types never exceeds 43.8. From the table, it can be seen that the number of iterations (column $\#iter$) is very closely related to the runtime (column $CPUtime_{iter}^{CBC}$).}

\begin{table*}[t]
\tiny
  \begin{center}
    \caption{Dependence of the runtime on the number of demands.}
    \label{tbl:BMscal}
    \begin{filecontents*}{BMscalCBC.csv}
,demand,runtimeGurobi,runtimeCBC,statusCBC,runtimeCBCiter,statusCBCiter,iterations,consB,consP
0,0-39,0.3,6.1,100,7.6,100,7.1,36.2,24.2
1,40-79,0.2,4.0,100,9.5,100,6.9,40.9,25.0
2,80-119,0.2,31.3,70,8.9,100,7.3,36.5,29.4
3,120-159,0.1,24.3,90,7.9,100,8.6,36.4,27.8
4,160-199,0.1,27.3,58,5.6,100,6.9,32.9,26.4
5,200-239,0.1,1.2,100,8.4,100,7.0,28.7,26.2
6,240-279,0.1,44.0,48,9.3,100,8.1,26.4,28.5
7,280-319,0.1,46.9,48,3.5,100,3.8,13.9,3.9
8,320-359,0.1,46.8,48,3.6,100,3.0,12.9,0.0
9,360-399,0.1,32.0,68,3.2,100,2.9,13.1,0.0
10,400-439,0.1,17.2,93,7.1,100,6.4,22.6,15.9
11,440-479,0.1,34.6,83,8.5,100,8.3,29.2,22.4
12,480-519,0.1,52.8,45,6.0,100,6.1,28.3,20.6
13,520-559,0.1,40.1,55,3.4,100,4.5,22.0,11.8
14,560-599,0.1,38.1,55,4.6,100,5.3,22.8,20.2
15,600-639,0.1,9.0,95,5.4,100,6.5,26.6,31.6
16,640-679,0.1,12.9,90,5.2,100,6.4,25.6,32.2
17,680-719,0.1,29.3,80,5.6,100,7.2,28.9,34.3
18,720-759,0.1,49.8,33,9.6,100,10.2,30.1,43.8
19,760-799,0.2,26.3,70,8.7,100,8.9,32.1,38.2
20,800-839,0.2,54.0,45,4.8,100,5.2,31.2,23.5
21,840-879,0.1,53.9,38,2.3,100,3.0,25.0,1.0
22,880-919,0.1,42.7,35,3.0,100,3.0,25.8,1.0
23,920-959,0.1,57.1,75,9.6,100,8.1,37.4,32.3
24,960-999,0.1,60.5,100,4.6,100,5.5,32.9,17.1
25,1000-1039,0.1,61.1,88,7.6,100,7.0,35.4,26.9
26,1040-1079,0.2,27.1,93,16.7,100,10.0,36.3,40.0
27,1080-1119,0.2,47.2,55,13.1,100,8.7,37.0,38.7
28,1120-1129,0.2,39.1,100,13.5,100,8.0,37.0,38.0
\end{filecontents*}
\pgfplotstabletypeset[
      columns={demand,runtimeGurobi,runtimeCBC,statusCBC,runtimeCBCiter,statusCBCiter,iterations,consB,consP},
      multicolumn names,
      every head row/.style={before row={\toprule}, after row=\midrule},
      col sep=comma,
      display columns/0/.style={column name=demand, column type=r, string type, column type/.add={}{|}},
      display columns/1/.style={column name=$CPUtime_{global}^{Gurobi}$, column type=r},
      display columns/2/.style={column name=$CPUtime_{global}^{CBC}$, column type=r, precision=1, fixed, fixed zerofill},
      display columns/3/.style={column name=$optimal_{global}^{CBC}$, column type=r, column type/.add={}{|}},
      display columns/4/.style={column name=$CPUtime_{iter}^{CBC}$, column type=r, precision=1, fixed, fixed zerofill},
      display columns/5/.style={column name=$optimal_{iter}^{CBC}$, column type=r},
      display columns/6/.style={column name=$\#iter$, column type=r},
      display columns/7/.style={column name=$\#consB$, column type=r},
      display columns/8/.style={column name=$\#consP$, column type=r},
      every last row/.style={after row=\bottomrule},
    ]{BMscalCBC.csv}
  \end{center}
\end{table*}


\textcolor{black}{
Nevertheless, the data from the company lead to very large ranges of values in the matrix coefficients, thereby causing numerical instability in Gurobi solver.} This is where the iterative model has a significant advantage. Since the global model introduces constraints (\ref{con:furbmina})--(\ref{con:furbminb}) also for materials that are purchased in large amounts, the constant $\mathcal{M}$ needs to be of very large value. Hence, to make the global model work correctly, we needed to set the value of $\mathcal{M}$ to $1\cdot 10^7$ in constraints (\ref{con:furbmina})--(\ref{con:furbminb}).
This, in turn, means that we also need to set the precision parameters of Gurobi ILP solver, such as \texttt{IntFeasTol}, to minimum possible values (which is $1\cdot 10^{-9}$), and \texttt{NumericFocus} to maximum value (which is 3), otherwise the solver fails to find a solution and incorrectly deems the instances infeasible. \textcolor{black}{On the other hand, this phenomenon was not observed with the iterative model, and the model works with the default parameters of the solvers.}

\section{Conclusion}\label{sec:conclusion}

The resilience of the meat production sector has been heavily tested in recent years by increasing prices of inputs and energy and fluctuating demand. In addition, due to EU regulations, meat production needs to be more efficient and environmentally friendly. Therefore, meat production companies must search for ways to react to this. This paper studies how the PPOP can improve production efficiency and increase the ability of production planners to adapt to fluctuating demand and other requests from production, e.g., the actual material availability in stock. The optimization problem described in this paper is derived from a real production process in the meat processing company. We proved the \mbox{$\mathcal{NP}$-completeness} of the problem, and we designed an ILP model integrating new constraints. The ILP model is solved using the iterative constraint generation approach, which allows solving the problem using an open-source ILP solver and positively impacts the numerical stability of the ILP solver. We performed an experimental analysis on real production data that showed the efficiency of our proposed algorithm as well as the effect of varying objective weights and extra requirements on the solution quality. The developed solution allows the company to optimize expenses related to the material used in production and analyze the optimal number of hogs to be cut. Furthermore, since the proposed algorithm can quickly solve the problem, the production planners can promptly adapt to changes in the demand for production or other disruptions. In future research, we want to concentrate on demand uncertainty. Specifically, we want to reduce the uncertainty budget by analyzing the demand for individual materials and the relation among them.

\section*{Acknowledgements}

This work was supported by the EU and the Ministry of Industry and Trade of the Czech Republic under the Project OP PIK CZ.01.1.02/0.0/0.0/20\_321/0024630, by the Grant Agency of the Czech Republic under the Project GACR 22-31670S, and co-funded by the European Union under the project ROBOPROX - Robotics and Advanced Industrial Production (reg. no. CZ.02.01.01/00/22\_008/0004590).


\bibliographystyle{itor}
\bibliography{ref}

\end{document}